\documentclass[conference]{IEEEtran}
\usepackage{times}

\usepackage[numbers]{natbib}
\usepackage{multicol}
\usepackage{amsmath}
\usepackage{amssymb}
\makeatletter
\@ifundefined{labelindent}{}{\let\labelindent\@undefined}
\makeatother
\usepackage{enumitem}
\usepackage{listings}
\usepackage{xcolor}
\usepackage{booktabs}
\usepackage[labelsep=colon, singlelinecheck=true, font={small}]{caption}
\usepackage[scale=0.9]{FiraMono}
\usepackage{svg}
\usepackage{amsthm}

\newtheorem{thm}{Theorem}
\newtheorem{prop}{Proposition}
\newtheorem{defn}{Definition}

\definecolor{codegreen}{HTML}{00997B} % Teal Green for strings
\definecolor{codegray}{rgb}{0.5,0.5,0.5}
\definecolor{codepurple}{rgb}{0.58,0,0.82}
\definecolor{backcolour}{HTML}{F0F0F0} % Slightly darker gray for better contrast
\definecolor{quartoblue}{HTML}{003B4F}
\definecolor{quartoteal}{HTML}{00769E} % Teal for decorators/types (was pink)
\definecolor{quartocomment}{HTML}{6A737D}
\definecolor{retrieverlink}{HTML}{1A5FB4} % Warm rust for refs
\definecolor{retrievercite}{HTML}{1A5FB4} % Blue-slate for citations
\definecolor{retrieverurl}{HTML}{1A5FB4} % Muted teal for URLs

\lstdefinestyle{retriever}{
    backgroundcolor=\color{backcolour},   
    commentstyle=\color{quartocomment}\itshape,
    keywordstyle=\color{quartoblue}\bfseries,
    keywordstyle=[2]\color{quartoteal}\bfseries, % Decorators and special types in Teal
    numberstyle=\tiny\color{codegray},
    stringstyle=\color{codegreen},
    basicstyle=\ttfamily\footnotesize,
    breakatwhitespace=false,         
    breaklines=true,                 
    captionpos=b,                    
    keepspaces=true,                 
    numbers=left,                    
    numbersep=10pt,                  
    showspaces=false,                
    showstringspaces=false,
    showtabs=false,                  
    tabsize=4,
    frame=lines, % Use lines or none, but need expansion
    rulecolor=\color{backcolour}, % Hide frame
    xleftmargin=20pt,
    xrightmargin=5pt,
    framexleftmargin=20pt,
    framexrightmargin=5pt,
    framextopmargin=5pt,
    framexbottommargin=5pt,
    aboveskip=1em,
    belowskip=1em,
    morekeywords={with, as, def, return, lambda, class, import, from},
    morekeywords=[2]{@Rate, @Hybrid, @Trigger, Pipeline, CameraFlow, BeliefFlow, GoalFlow, MonitorFlow, PlanFlow, SkillFlow, ControllerFlow, TeleopFlow, CameraSource, BeliefMemoryFlow, LanguageGoalFlow, ExecutionMonitorFlow, VLMPlanFlow, VLAFlow, IKControllerFlow, Window, Latest, Trigger},
    literate={_}{\_}{1}
}
\usepackage[T1]{fontenc}
\usepackage[bookmarks=true]{hyperref}
\usepackage{xspace}
\usepackage{changepage} % narrow appendix measure
\usepackage[disable]{todonotes}
\let\oldtodo\todo
\newcommand{\todoP}[2]{%
    \ifcase#1\oldtodo[color=red!15]{P0: #2}%
    \or\oldtodo[color=orange!20]{P1: #2}%
    \or\oldtodo[color=purple!15]{P2: #2}%
    \else\oldtodo[color=orange!20]{P#1: #2}%
    \fi
}
\newcommand{\todoInlineP}[2]{%
    \ifcase#1\oldtodo[color=red!15, inline]{P0: #2}%
    \or\oldtodo[color=orange!20, inline]{P1: #2}%
    \or\oldtodo[color=purple!15, inline]{P2: #2}%
    \else\oldtodo[color=orange!20, inline]{P#1: #2}%
    \fi
}
\renewcommand{\todo}[1]{\todoP1{#1}}

\newcommand{\asyncmodel}{asynchronous environment--agent loop\xspace}

\newif\ifshowtodolist
\showtodolistfalse
\newcommand{\maybeTodoList}{%
    \ifshowtodolist
        \clearpage
        \section*{Todo List}
        \listoftodos
        \clearpage
    \fi
}

\hypersetup{
    pdftitle={Retriever: Composing the Perception-Reasoning-Action Loop for Long-Horizon Manipulation},
    pdfauthor={Linfeng Zhao, Haojie Huang, Jiayuan Mao, Weiyu Liu, Lawson L.S. Wong, Mykel J. Kochenderfer},
    pdfkeywords={Robots, Programming Model},
    pdfsubject={Robots},
    colorlinks=true,
    linkcolor=retrieverlink,
    citecolor=retrievercite,
    urlcolor=retrieverurl,
    pdfborder={0 0 0}
}

\title{Retriever: Composing the Perception-Reasoning-Action Loop for Long-Horizon Manipulation}

\author{
{Linfeng Zhao}{${^\P}$},
{Haojie Huang}{${^\ddagger}$},
{Jiayuan Mao}{${^\S}$},
{Weiyu Liu}{${^\P}$},
{Lawson L.S. Wong}{${^{\star\ddagger}}$},
{Mykel J. Kochenderfer}{${^{\star\P}}$}
\\
\vspace{0.2cm}
{${~^\P}$}Stanford University,
{${~^\S}$}MIT,
{${~^\ddagger}$}Northeastern University,
{${~^\star}$}Equal Advising
}

\begin{document}

\maketitle
\pagestyle{plain}

\IEEEpeerreviewmaketitle

% !TEX root = ../main_arxiv.tex
\begin{abstract}
Building long-horizon robot agents requires composing closed-loop pipelines---perception, belief update, planning, and control---whose components run at different clocks and with variable latency.
Today, these systems are often assembled with ad-hoc concurrency and pub/sub conventions that make timing and input-consumption semantics implicit, yielding schedule-dependent behavior that is hard to reproduce, debug, and reuse.
Current solutions typically solve parts of this problem at either the algorithmic or the systems layer, but not both.
In this work, we propose Retriever, which spans the entire stack: an asynchronous decision model, a programming model, a runtime, and an example closed-loop agent pipeline.
Retriever represents an agent as a graph of stateful \textit{causal stream functions} executed on explicit run clocks.
We formalize this view via an asynchronous environment--agent loop over continuous-time streams and show that finite-memory causal policies can be represented by compositions of these operators.
Retriever compiles these graphs to multiple backends and provides a deterministic replay contract when consumed inputs, clock/commit traces, and stochastic outputs are logged.
We evaluate Retriever through a real-robot case study together with controlled studies of runtime overhead and deterministic replay behavior.
\end{abstract}
 % only abstract here

% !TEX root = ../main_arxiv.tex
\section{Introduction}
\label{sec:intro}

% Q1: What is the problem?
General robotic manipulation poses fundamental challenges. Tasks often require long-horizon decision making under partial observability while following open-ended human instructions in dynamic environments.
Training a single unified policy is often infeasible because of data or latency constraints.
Instead, effective agents use \emph{closed-loop, reactive, modular pipelines} (e.g., Figure~\ref{fig:case-study-pipeline}) in which specialized modules handle perception, reasoning, and control at different rates.
Even as data scaling gradually pushes the field toward more integrated systems, this modular stage remains inevitable.

However, today's stacks offer no general-purpose way to build such pipelines: each system is engineered as a brittle one-off, even though the same coordination problems recur across tasks and platforms.
If the pipeline blocks on a slow model call, control stalls and the robot stutters.
If modules run without coordination, they act on stale data, and behavior depends on thread scheduling in ways that are nearly impossible to debug or replay.
While deep learning thrives on modularity and compositionality (e.g., PyTorch layers), robotics does not yet have a widely adopted abstraction that makes the temporal composition of these components explicit and deterministic.
Such an abstraction is increasingly timely: general-purpose coding agents can already write and debug programs together with roboticists, making programs the key intermediate representation between foundation models and robots.

%%%%%%%%%%%%%%%%%%%%%%%%%%%%%%%%%%%%%%%%%%%%%%%%%%
% Don't change this part

\begin{figure}[t]
    \centering
    \begin{lstlisting}[style=retriever, language=Python]
from retriever import Pipeline, Rate, Trigger, Latest

# 1) Define Flows (what) + run clocks (when)
head_cam = CameraSource(id=0)     @Rate(hz=30)
wrist_cam = CameraSource(id=1)    @Rate(hz=30)
belief  = BeliefMemoryFlow()      @Trigger(
    "inspection_done")
monitor = ExecutionMonitorFlow()  @Trigger(
    "belief_updated", "progress")
planner = VLMPlanFlow("gemini") @Trigger("replan")
vla     = VLASkillFlow("pi05")    @Rate(hz=2)
robot   = ControllerFlow(id=0)    @Rate(hz=200)

# 2) Build the graph (composition + sync on edges)
pipe = Pipeline("Closed-loop Agent")
with pipe:  # Each line defines a chain
    wrist_cam.then(vla, sync=Latest())\
             .then(robot, sync=Latest())
    head_cam.then(belief, sync=Latest())\
            .then(planner, sync=Latest())\
            .then(monitor, sync=Latest())\
            .then(vla, sync=Latest())

# 3) Debug vs deploy
pipe.step(dt=0.1)  # Run for 0.1s in main process
pipe.run(backend="dora")  # Run async -> deploy
\end{lstlisting}
    \vspace{-0.5em}
    \caption{\textit{\texttt{Retriever-0} (representative pipeline).} Retriever is a programming model and runtime framework for closed-loop reactive agents. Agents are written by composing Flows into computational graphs (via \texttt{then}) and making time explicit: each Flow declares a run clock, and each edge declares a synchronization or buffering contract. Between Flows with mismatched rates, plan chunks and action chunks carry time-extended outputs, so slow planning overlaps skill execution and 2Hz inference keeps 200Hz control fed (Sec.~\ref{sec:case_study}). The graph can be stepped in the main process for debugging (\texttt{.step()}) or executed asynchronously (\texttt{.run()}). The listing is essentially runnable Python, simplified mainly in naming, and corresponds to the canonical pipeline in Fig.~\ref{fig:case-study-pipeline}.}
    \label{fig:teaser}
    \vspace{-1.8em}
\end{figure}
%%%%%%%%%%%%%%%%%%%%%%%%%%%%%%%%%%%%%%%%%%%%%%%%%%

%%%%%%%%%%%%%%%%%%%%%%%%%%%%%%%%%%%%%%%%%%%%%%%%%%
% \begin{figure*}[!t]
%     \centering
%     \textcolor{lightgray}{\rule{\linewidth}{6cm}}
%     \caption{\textbf{[TODO placeholder: multiple examples] The Flow Graph Diagram.}}
%     \label{fig:case-study-diagram}
% \end{figure*}
\begin{figure*}[t]
    \centering
    \includegraphics[width=0.9\linewidth]{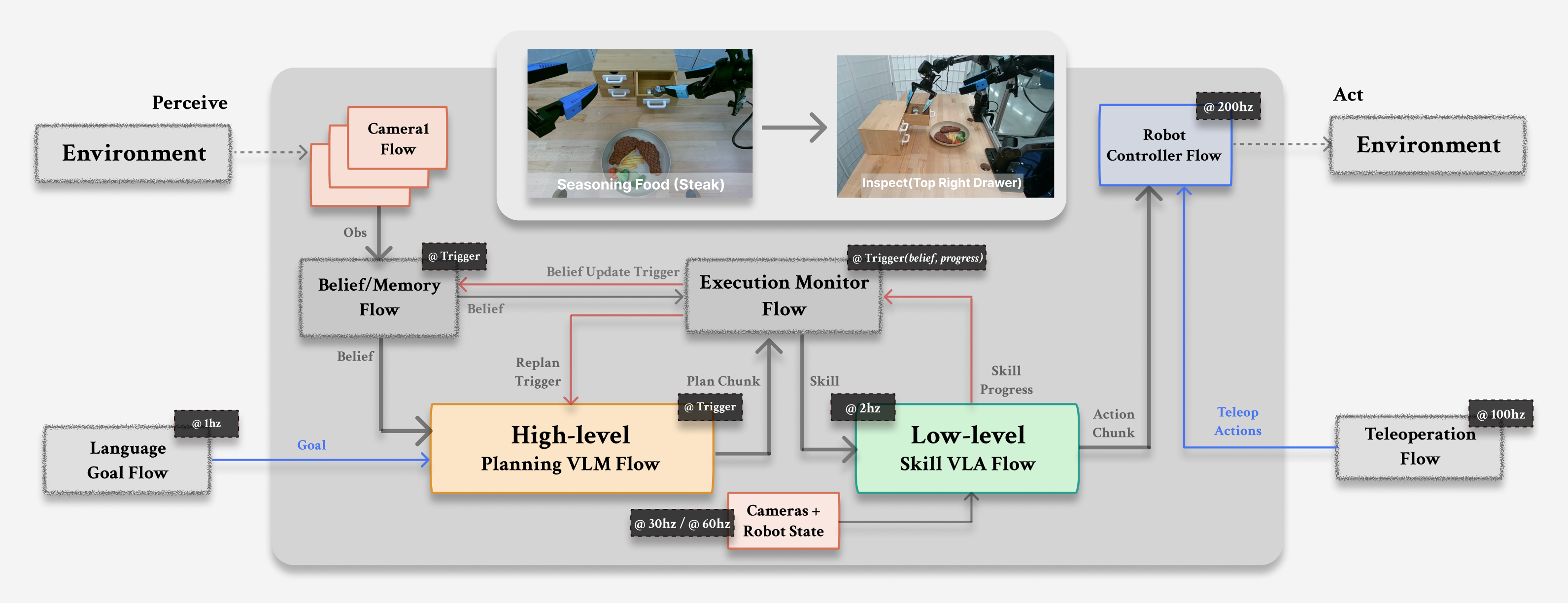}
    \caption{\textbf{System Pipeline.} \texttt{Retriever-0} composes camera, belief/memory, planning VLM, execution-monitor, VLA skill, and robot-controller Flows. Triggered belief and progress updates close the loop; \texttt{PlanChunk}s and \texttt{ActionChunk}s bridge slow planning, 2Hz skill inference, and 200Hz control.}
    \label{fig:case-study-pipeline}
    \vspace{-1.5em}
\end{figure*}

%%%%%%%%%%%%%%%%%%%%%%%%%%%%%%%%%%%%%%%%%%%%%%%%%%

% Q2: Why is it interesting and important?
Tightly coupling semantic reasoning with reactive skills is central to robust autonomy.
For instance, a tool-use task (Fig.~\ref{fig:case-study-pipeline}) may combine slow VLM planning, fast VLA skill execution, real-time failure monitoring, and still faster inverse-kinematics and control loops.
These modules must remain coordinated despite their different rates.\footnote{Language-goal and teleoperation Flows are optional input sources for task specification, recovery, and data collection.}
A programming model that makes these temporal couplings explicit would allow developers to safely compose modules at mismatched rates without relying on brittle, task-specific glue code.

% Q3: Why is it hard? (E.g., why do naive approaches fail?)
Real-world robotic systems operate \textbf{asynchronously and under strict timing constraints}, creating a fundamental mismatch with our abstractions.
The challenge is twofold.
(1) \textbf{Physics vs.\ Compute:} Large models (e.g., VLMs) have variable latency while physical control requires strict deadlines. Standard Markov decision process (MDP/POMDP) formulations abstract away execution time. In a naive implementation, blocking stalls execution while non-blocking access can yield stale decisions.
(2) \textbf{Determinism vs.\ Asynchrony:} Many middleware applications (e.g., ROS stacks) rely on implicit callback ordering. The resulting schedule dependence makes behavior difficult to replay or verify and makes differentiable optimization unreliable.

% Q4: Why hasn't it been solved before? (Or, what's wrong with previous proposed solutions? How does mine differ?)
% \noindent
A natural goal is to decouple module algorithms from asynchronous execution details.
Existing paradigms provide \textbf{limited support} for this in \textbf{closed-loop} settings (Table~\ref{tab:comparison}).
(1) \textbf{Mismatched Abstractions:} Learning formulations (e.g., RL and imitation learning over MDPs/POMDPs) assume a globally synchronized timestep, while systems tools (e.g., ROS) assume loose message passing, losing temporal semantics.
(2) \textbf{Missing Abstraction Boundary:} Middleware lacks explicit primitives for temporal coordination (buffering, chunking), leaving correctness to ad-hoc glue code.
In this paper, we introduce Retriever, \textbf{a new programming model and a flexible runtime} for asynchronous, compositional robot agents.
Its primitives make buffering, scheduling, and deadlines explicit without embedding composition logic inside individual modules.
The result is \emph{functional determinism}: under the fixed clocks, synchronization policies, and replay conditions formalized later, the same input trace yields the same output trace regardless of scheduler interleavings.
This supports reproducible replay, verification, and path-gradient studies over fixed temporal traces.

\begin{table*}[t]
    \centering
    \scriptsize
    \begin{tabular}{@{}l|llll@{}}
    \toprule
    \textbf{Framework} & \textbf{Abstraction} & \textbf{Time Model} & \textbf{Determinism} & \textbf{Primary Scope} \\
    \midrule
    Standard RL (Gym) & MDP/POMDP/Agent Loop & Synchronized Steps & Sequential & Algorithmic Learning \\
    \midrule
    Imperative & Call-Return / Loop & Implicit (Blocking) & Sequential & Prototyping \& Glue Code \\
    ROS / ROS 2 & Pub/Sub (Transport) & Implicit (Wall-clock) & Best-effort* & Async Robot System \\
    Process networks (KPN) & Process Graph + FIFOs & Logical Streams & Schedule-independent & Stream Semantics \\
    FRP (e.g. Yampa) & Signal Functions & Continuous / Logical & Functional Determinism & Functional Logic \\
    Actor model / Ray & Actors + Async Messages & Event-driven & Nondeterministic & Distributed Execution \\
    TF / PyTorch (ML) & Graph IR (+ runtime) & Discrete Steps & Functional Determinism & Gradient Optimization \\
    \midrule
    \textbf{Retriever} & {\textbf{Behavior-Level} Graph} & {\textbf{Explicit} (Clocks/Sync)} & {Functional Determinism} & {\textit{Closed-Loop} Agent Composition} \\
    \bottomrule
    \end{tabular}
    \caption{Comparison of Retriever with related paradigms. Retriever targets a systems-level abstraction that makes time and input-consumption semantics explicit for deterministic closed-loop agent composition; see Appendix~\ref{app:history} for a longer historical lineage, and Table~\ref{tab:design-goals} (Sec.~\ref{sec:framework}) for how each family covers Retriever's design desiderata. *Best-effort: not guaranteed schedule-independent; \emph{functional determinism}: outputs are functions of input histories, not scheduler interleavings (Sec.~\ref{sec:related_work}).}
    \label{tab:comparison}
    \vspace{-1.5em}
\end{table*}

% Q5: What are the key components of my approach and results? Also include any specific limitations.
Our contributions include:
(1) \textbf{Formulation in Physical Time:} We formulate asynchronous agent--environment interaction over continuous physical time, without the single global timestep assumed in MDPs and POMDPs (Sec.~\ref{sec:formulation}).
(2) \textbf{Programming Model:} We define a dataflow graph in which modules have explicit clocks and edges use \emph{synchronization policies} for deterministic input consumption.
(3) \textbf{Runtime Framework:} Our embedded Python DSL compiles to an intermediate representation that can be stepped locally for debugging and replay or run on supported backends such as Dora without replacing transport middleware (Sec.~\ref{sec:framework}).
(4) \textbf{Long-Horizon Manipulation Pipeline and Controlled Studies:} We build a closed-loop multi-rate pipeline for long-horizon manipulation that combines VLM planning, reactive skills, belief memory, and monitoring, together with controlled studies of runtime overhead and determinism (Sec.~\ref{sec:case_study}, \ref{sec:eval}).

% !TEX root = ../main_arxiv.tex
\section{Related Work}
\label{sec:related_work}

Retriever combines ideas from robotic systems, dataflow/stream processing, and functional programming, including the \textit{Unix philosophy} of simple, composable interfaces~\cite{Ritchie1974UNIX}.
Existing tools address connectivity, modularity, and reactivity from different layers.
It unifies these concerns under one timing model.
Table~\ref{tab:comparison} compares the underlying paradigms, and Table~\ref{tab:design-goals} (Sec.~\ref{sec:framework}) scores their coverage of the design desiderata.
% DONE (2026-05-06): citation/venue metadata validated and fixed.

\subsection{Open-World Robot Agents}
Open-world robot agents must solve long-horizon tasks under partial observability.
They must gather information, maintain belief and task progress, execute skills, monitor outcomes, and replan after failures.
Modular systems address these needs through STRIPS-style decoupling of reasoning and execution~\cite{Fikes1971STRIPS}, hierarchical task planning and skill libraries~\cite{Kaelbling2011HTN}, compositional skill models for task-and-motion planning~\cite{Wang2021Compositional}, and object-oriented world models~\cite{Zhao2022HOWM}.
Belief-space planning formalizes the same idea by aggregating history into an explicit belief state for planning-to-perceive~\cite{kaelbling2012unifying, garrett2020online, zhao2025seeing} and uncertainty handling~\cite{bonet2011, curtis2024partially}.
Language-model-driven instruction-following systems~\cite{Ahn2022SayCan} adopt a similar modular structure (estimation, deliberation, execution), but amplify timing challenges due to variable-latency inference.
% Our takeaway from this literature is that the \textit{modular pipeline skeleton} is stable, and it naturally induces multiple interacting closed loops (fast skill/control, slower belief/planning, and the agent--environment loop); what is missing is a \textit{systems-level} abstraction that makes temporal coordination between modules explicit and reproducible under asynchrony.
%
Interfaces such as \texttt{policy.forward($\cdot$)} or a single LLM call do not provide a programming model for coordinating closed-loop asynchronous modules in embodied environments.
The missing piece is explicit temporal semantics for module latency, multi-rate coordination, and delay handling.
% Our work aims to find a programming model that is \textit{expressive} enough to build a closed-loop hierarchical agent with high-level \textit{planning in \textit{belief space}} plus low-level policy execution (\citep{kaelbling2012unifying,zhao2025seeing}), where we find a primary missing piece is the consideration of temporal aspects: how long does it take to compute, how to coordinate between asynchronous modules, and how to handle wait/delay. 
%%%
% \todo{this could be compressed}
%%%
% Retriever represents the pipeline as a time-aware dataflow graph of Flows (stateful causal stream functions), where run clocks and edge synchronization policies are explicit and inspectable.

\subsection{Robotics Middleware and Frameworks}
Robotics systems are commonly built atop message-passing middleware such as ROS~\cite{Quigley2009ROS} and ROS 2~\cite{Macenski2022ROS2}.
These systems provide connectivity and deployment tooling but largely leave timing and data alignment to the application.
Callbacks, queues, and shared state therefore encode \textit{when} modules run and \textit{what} they consume, including alignment, bounded staleness, and multi-rate dependencies.
These hidden policies make runs hard to reproduce and input consumption hard to inspect.
A carefully engineered ROS 2 stack can realize many of these behaviors case by case, but as bespoke application code rather than reusable contracts.
More recent systems such as Dora~\cite{dora} provide execution and transport for robotics graphs.
Retriever sits above this layer and specifies graph timing, synchronization, and input-consumption contracts.
Real-time frameworks like OROCOS~\cite{Bruyninckx2001OROCOS} and Cyber RT~\cite{Baidu2019Cyber} strengthen guarantees for safety-critical loops, but do not provide a reusable abstraction for heterogeneous agent graphs spanning slow deliberation and fast control.
This motivates a programming layer that exposes timing and input-consumption semantics at the graph level (Sec.~\ref{sec:intro}--\ref{sec:framework}).

\subsection{Programming Models for Asynchronous systems}
A \textit{dataflow graph model} represents programs as graphs of stateful operators connected by streams.
Kahn process networks (\textit{KPN}) provide a key anchor: each node is a deterministic sequential process communicating only via FIFO channels, where reads block when inputs are unavailable~\cite{Kahn1974Semantics}.
Under these assumptions, the entire network is \textit{determinate}: for a fixed program and fixed input histories, the histories produced on all channels are uniquely determined and do not depend on the runtime schedule.
We refer to this property as \textit{functional determinism}---outputs are a function of input histories, not of thread interleavings.
Synchronous dataflow further restricts rates to enable static schedules and resource analysis~\cite{Lee1987SDF}.

In parallel, \textit{functional programming} and \textit{functional reactive programming (FRP)} emphasize explicit semantics for composing time-varying computation and feedback~\cite{Backus1978VonNeumann, Hughes1989WhyFP}.
\textit{Arrowized FRP} formalizes reactive systems as compositions of \textit{causal signal functions} with disciplined feedback~\cite{Nilsson2002FRP, Hughes2000Arrows}.
For embodied robotics, these models still leave two choices implicit: how a module aligns multi-rate event streams at decision time, and how that choice is inspected and replayed across runtimes.

The \textit{actor model} emphasizes isolated state with asynchronous message passing~\cite{Hewitt1973Actor, Armstrong2003Erlang}; modern runtimes such as Ray operationalize similar ideas for AI workloads~\cite{Moritz2018Ray}.
However, actors alone do not specify a unique input-consumption order across concurrent arrivals, so determinism typically requires additional constraints.
Finally, \textit{ML frameworks} such as TensorFlow and PyTorch use explicit graphs/IRs for portability and optimization~\cite{Abadi2016TensorFlow, Paszke2019PyTorch}, but largely assume synchronous tensor steps rather than asynchronous sensorimotor streams.

Retriever combines these ideas in an explicit agent graph whose input consumption is schedule-independent and anchored in physical time.
Appendix~\ref{app:history} provides a longer historical lineage.

% !TEX root = ../main_arxiv.tex
\section{Formulation: Asynchronous Environment--Agent Loop}
\label{sec:formulation}

% Future figure note: a side-by-side "pub/sub arrival-time race" vs.
% "explicit sync snapshot" diagram could strengthen this section if space permits.
% Earlier placeholder figure is recoverable from git history before this concision pass.

We first define time-varying values and the causal transformations between them.
We use the formal term \textit{causal stream function} (CSF) for the semantic object; the programming object introduced later is a \texttt{Flow}.
We use these definitions to lift the standard agent--environment loop to continuous time.
Observations, computation completions, and actions become streams over event time rather than values locked to one global step.
Sec.~\ref{sec:framework} turns this timing model into a programming model and runtime.
Table~\ref{tab:design-goals} states the design desiderata used to evaluate it.
Appendix~\ref{app:formulation_ext} and Appendix~\ref{app:formulation_related} provide additional definitions and timing discussion.

\subsection{The Data Types: Time and Streams}
We assume a global continuous time-tag set $\mathbb{T} = \mathbb{R}_{\ge 0}$ shared by all streams, at which sensors fire and actions take effect.

For a value type $\mathcal{V}$, we write $\mathsf{Stream}(\mathcal{V})$ for the set of time-tagged streams taking values in $\mathcal{V}$.
Following the tagged-signal view~\citep{Lee1998MoC,Liu2005TaggedSignals}, a \textbf{Stream} (or \textbf{Signal}) is a partial function $s:\mathbb{T}\rightharpoonup\mathcal{V}$.
Equivalently, it is a set of timestamped events $\{(t_i,v_i)\}$ with at most one value per time tag in a given stream.
We write events as time--value pairs $(t_i,v_i)$ throughout the paper.
We use $s(t)=\bot$ only as notation for ``no event at time $t$.''
There are two fundamental types of stream.
\begin{itemize}
    \item \textbf{Continuous streams (behaviors):} $s_c : \mathbb{T} \to \mathcal{V}_c$. Defined for all $t$ (e.g., physical properties, robot joint positions, velocity).
    \item \textbf{Discrete event streams:} $s_e:\mathbb{T}\rightharpoonup\mathcal{V}_e$ with countable, locally finite support (e.g., camera images, robot actions, messages, contact events).
\end{itemize}
This partial-function notation covers continuous behaviors and discrete events within one causal model.
Related signal/function views appear in FRP and synchronous-language semantics~\citep{Wan2000FRP,Nilsson2002FRP}.
For any stream $s$, we write $s|_{\le t}$ or $s|_{<t}$ for the restriction of the stream to tags up to, or strictly before, time $t$.

We define a \textbf{Clock} as a special Event Stream whose values are unitless ``ticks.''\footnote{Equivalently, we identify a clock with its locally finite support of trigger times $\mathrm{supp}(\mathcal{C})=\{t_k\}$.}
A clock $\mathcal{C}$ defines when a computational module runs, effectively discretizing continuous time for function execution.
Clocks can be periodic, event-driven, a union of both, or driven by external synchronization signals.
Event time remains continuous; clocks induce local decision epochs, so composition yields interleaved local steps rather than one global step sequence.

\subsection{\texorpdfstring{Causal Stream Functions (CSFs) and Flows ($\mathcal{F}$)}{Causal Stream Functions and Flows}}
The semantic unit of computation is a \textbf{causal stream function} (CSF) $\mathcal{F}: \mathsf{Stream}(\mathcal{I}) \to \mathsf{Stream}(\mathcal{O})$, which maps input streams of type $\mathcal{I}$ to output streams of type $\mathcal{O}$.\footnote{For a multi-input module with port value types $\mathcal{I}_1,\dots,\mathcal{I}_m$, $\mathsf{Stream}(\mathcal{I})$ is shorthand for the tuple of independent input streams $\prod_{j=1}^m \mathsf{Stream}(\mathcal{I}_j)$; the synchronized input at a tick is the port tuple $x_t\in\prod_{j=1}^m\mathcal{I}_j$.}
Let $s_{\mathrm{in}}$ and $s'_{\mathrm{in}}$ denote two input streams. Causality means that if they agree through time $t$, their output histories also agree through $t$:
\begin{equation}
    s_{\mathrm{in}}|_{\le t}=s'_{\mathrm{in}}|_{\le t}
    \;\Longrightarrow\;
    \mathcal{F}(s_{\mathrm{in}})|_{\le t}
    =\mathcal{F}(s'_{\mathrm{in}})|_{\le t}.
\end{equation}
This definition lets a pointwise map use an event at $t$.
Feedback cycles require the stricter prefix rule introduced below.

Two properties make CSFs modular.
First, CSFs are \textbf{stateful}: $\mathcal{F}$ can maintain internal memory $h_t$ (e.g., integrators, Kalman filters, belief states) that evolves over time.
Second, they satisfy \textbf{closure}: a graph of CSFs with strictly causal feedback, as specified below, is \textit{itself} a CSF.
This lets us build agent hierarchies from submodules without reasoning about one global state.
At an event time $t$, a stateful CSF can be written as a local transition
\begin{equation}
    (h_{t^+}, y_t) = f(h_{t^-}, x_t),
\end{equation}
where $h_{t^-}$ and $h_{t^+}$ are the local state immediately before and after processing the event, $x_t$ is the synchronized input snapshot consumed by the step, and $y_t \in \mathcal{O}$ is the emitted output value.
Closed-loop composition makes these graphs \emph{cyclic}: feedback between acting and sensing is the point, not an edge case.
Instantaneous cycles would impose algebraic constraints such as $x_t=f(x_t)$.
Therefore, every directed cycle must contain at least one delay or \textbf{strictly causal} edge whose consumer sees only producer events from times $<t$.\footnote{A state-delay node's output at tick $t_k$ depends only on state established before $t_k$. Acyclic same-tick dependencies are also admissible when their order is declared deterministically.}
Unlike a deep-learning graph that can avoid feedback or unroll recurrence offline, a robot agent must execute such feedback directly in time.

To execute a CSF, we bind it to a \textbf{Clock} and input-snapshot rules.
The resulting runtime object is a \textbf{Flow}: the CSF defines \emph{what} is computed, while the Flow defines \emph{when} it runs and which aligned snapshot it consumes.
\begin{equation}
    \texttt{Flow}(\mathcal{F}, \mathcal{C}) : \mathsf{Stream}(\mathcal{I}) \to \mathsf{Stream}(\mathcal{O})
\end{equation}
Here $\mathcal{I}$ and $\mathcal{O}$ are the input and output value spaces.
At the $k$-th sample tick $t_k \in \mathrm{supp}(\mathcal{C})$, the Flow evaluates its local transition.
Edgewise \texttt{sync} policies assemble the synchronized input snapshot $x_{t_k}$ (Sec.~\ref{sec:framework}):
\begin{equation}
    (h_{t_k^+}, y_{t_k}) = f(h_{t_k^-}, x_{t_k}).
\end{equation}
A Flow runs at $t_k$ and makes its result visible at a declared commit time $c_k$ (usually $c_k=t_k$), so periodic and event-driven modules share the same causal semantics.\footnote{Variable-latency policies may use $c_k>t_k$. Replay fixes or logs both $\{t_k\}$ and $\{c_k\}$, and downstream \texttt{sync} policies read only committed events.}

\subsection{Asynchronous Environment--Agent Loop}
\label{subsec:async_interaction_sync}

We extend the standard MDP/POMDP agent--environment loop by relaxing its single global timestep.
In the resulting \asyncmodel, observations and actions are causal streams over continuous event time.
The task objective is unchanged.
The model instead specifies which past values a module may consume when modules observe, compute, and act at different rates.
At the task level, the important split is:
\begin{enumerate}
    \item \textbf{Environment $\mathcal{E}$:} A causal map from action streams to observation streams. The environment evolves continuously, and observations may arrive at any time.
    \item \textbf{Agent Policy $\pi$:} A causal stream function mapping observation history to action streams. The agent does not wait for a ``next step''; it reacts asynchronously to events.
\end{enumerate}
At a local tick $t_k$, a \texttt{sync} policy selects an aligned snapshot or buffer from upstream events before $t_k$.
The runtime assembles these edge values into the module's ordered input record.
Sec.~\ref{sec:framework} implements this as \texttt{sync(...)} calls between clocked Flows.
Unlike step-based MDPs, the environment does not wait: while a module computes, physical state and other streams continue to evolve.
The informational constraint remains causal: an emitted action can depend only on observations committed to the histories sampled by its synchronization policies.

% Reference-only note: an earlier commented VLM planner Flow pseudocode block
% lived here during drafting. Keep concrete Flow code examples in Sec.~\ref{sec:framework}
% and the case-study appendix unless we decide this section needs another listing.
% The old block is recoverable from git history before this concision pass.

% \note{If space permits, we can upgrade this to a side-by-side ``arrival-time pub/sub'' vs. ``explicit sync'' comparison to make the source of nondeterminism visually explicit.}

%%%%%%%%%%%%%%%%%%%%%%%%%%%%%%%%%%%%%%%%%%%%%%%%%%

% !TEX root = ../main_arxiv.tex
\section{Retriever Programming Model \& Runtime}
\label{sec:framework}

This section turns the causal-stream-function (CSF) formulation into a programming model and runtime for multi-rate closed-loop systems (Fig.~\ref{fig:case-study-pipeline}).
Temporal coordination becomes an explicit graph contract rather than glue code.

%%%%%%%%%%%%%%%%%%%%%%%%%%%%%%%%

%%%%%%%%%%%%%%%%%%%%%%%%%%%%%%%%

\subsection{Design Desiderata}
We start from requirements induced by the canonical pipeline in Fig.~\ref{fig:case-study-pipeline}:
\begin{itemize}[leftmargin=*]
    \item \textbf{D1} \textit{Compositionality:} systems can be built by composing reusable modules that depend only on their inputs and local state, not global state and glue logic inside callbacks.
    \item \textbf{D2} \textit{Asynchronous multi-rate:} planners, skills, and controllers run in parallel at different rates, and model inference has variable latency; there is no single global synchronized tick.
    \item \textbf{D3} \textit{Closed-loop modularity:} agent pipelines need closed-loop feedback (monitoring $\rightarrow$ replanning $\rightarrow$ execution) without ad-hoc concurrency.
    \item \textbf{D4} \textit{Explicit time semantics:} slow$\rightarrow$fast and fast$\rightarrow$slow boundaries require principled handling of synchronization and buffering (e.g., staleness, chunking, and timeouts).
    \item \textbf{D5} \textit{Determinism, replay, and debuggability:} same inputs and initial internal state should yield identical traces/outputs under replay (independent of runtime scheduling), enabling data collection and debugging with breakpoints.
    \item \textbf{D6} \textit{Supported backend mapping:} decouple graph semantics from placement and transport choices where backend adapters exist, without treating hardware-specific behavior as part of the algorithm.
    % \item \textit{Debuggability \& data:} logs should answer ``what input snapshot did this output depend on?'' and support dataset extraction.
\end{itemize}
Table~\ref{tab:design-goals} scores how representative abstraction families cover these desiderata.

\begin{table*}[t]
    \centering
    \scriptsize
    \setlength{\tabcolsep}{3pt}
    \renewcommand{\arraystretch}{1.08}
    \begin{tabular}{@{}p{2.85cm}cccccc@{}}
    \toprule
    \textbf{Framework family} &
    \textbf{\begin{tabular}[c]{@{}c@{}}Async\\closed loop\\{\tiny(D2--D3)}\end{tabular}} &
    \textbf{\begin{tabular}[c]{@{}c@{}}Modular\\graph\\{\tiny(D1)}\end{tabular}} &
    \textbf{\begin{tabular}[c]{@{}c@{}}Explicit\\sync/stale\\{\tiny(D4)}\end{tabular}} &
    \textbf{\begin{tabular}[c]{@{}c@{}}Temporal\\skills/chunks\\{\tiny(D4)}\end{tabular}} &
    \textbf{\begin{tabular}[c]{@{}c@{}}Replayable\\step inputs\\{\tiny(D5)}\end{tabular}} &
    \textbf{\begin{tabular}[c]{@{}c@{}}Backend\\mapping\\{\tiny(D6)}\end{tabular}} \\
    \midrule
    MDP/POMDP loop~\cite{Puterman1994MDP} & Partial & No & No & Partial & Partial & No \\
    ROS 2/Dora~\cite{Macenski2022ROS2,dora} & Partial & Partial & Manual & Manual & Partial & Partial \\
    TAMP/options~\cite{kaelbling2012unifying,Sutton1999Options} & Partial & Partial & No & Yes & No & No \\
    KPN/FRP~\cite{Kahn1974Semantics,Nilsson2002FRP} & Partial & Yes & Partial & No & Yes & No \\
    Ray/actors~\cite{Hewitt1973Actor,Moritz2018Ray} & Partial & Yes & No & No & Partial & Yes \\
    PyTorch/TF IRs~\cite{Paszke2019PyTorch,Abadi2016TensorFlow} & No & Yes & No & No & Partial & Yes \\
    Retriever & Yes & Yes & Yes & Yes & Yes & Yes \\
    \bottomrule
    \end{tabular}
    \caption{Coverage of the design desiderata \textbf{D1--D6} across representative abstraction families; Table~\ref{tab:comparison} gives the corresponding paradigm taxonomy. ``Yes'' means the capability is first-class in the abstraction; ``Partial'' means the abstraction supports part of the need but leaves an important robot-agent contract unspecified; ``Manual'' means the behavior is implementable but encoded in application glue; ``No'' means it is outside the abstraction's target. Retriever's contribution is combining these capabilities in a single abstraction.}
    \label{tab:design-goals}
    \vspace{-1.0em}
\end{table*}

\subsection{Core Abstractions: Flows, Clocks, and Pipeline Graphs}
Retriever programs the \emph{agent computation graph}: users write \textbf{Flows} and connect them inside a \textbf{Pipeline}.
A Flow is a stateful synchronous step function (Fig.~\ref{fig:vla_flow_example}); asynchrony comes from its clock and graph edges. Programmatically, each Flow realizes a CSF from Sec.~\ref{sec:formulation} with three fields:
\begin{itemize}[leftmargin=*]
    \item \texttt{step()} --- the synchronous transition: update local state, emit outputs;
    \item \texttt{clock} --- \emph{when} the Flow steps: a periodic rate or an event trigger;
    \item input ports --- \emph{what} each step consumes: typed edges whose histories are materialized by \texttt{sync} policies.
\end{itemize}

\begin{figure}[t]
    \centering
\begin{lstlisting}[style=retriever, language=Python]
# Example: a VLA skill flow @2Hz
# (wrist cam obs, robot state, SkillCmd) -> ActionChunk
class VLASkillFlow(Flow):
    def reset(self):
        self.cur_skill = None

    def step(self, inp):
        cam_obs, robot_state, skill_cmd = inp  # aligned by sync policies
        self.cur_skill = skill_cmd
        return self.vla(cam_obs, robot_state, self.cur_skill)
\end{lstlisting}
    \caption{\textit{VLA Flow Example.} At each 2Hz \texttt{step()}, the skill Flow receives a synchronized snapshot: the latest 30Hz wrist-camera frame, 60Hz robot state, and active \texttt{SkillCmd}. It emits an \texttt{ActionChunk} so high-rate control can continue between policy calls (Fig.~\ref{fig:async-vs-sync-comparison}).}
    \label{fig:vla_flow_example}
    \vspace{-1.0em}
\end{figure}
These choices directly reflect the desiderata above: Flows provide a reusable, state-local unit of composition (\textbf{D1}); clocks make multi-rate execution explicit (\textbf{D2}); and edge policies make cross-rate interaction deterministic and replayable (\textbf{D4--D5}) while remaining portable across backends (\textbf{D6}). Closed-loop wiring (monitoring $\rightarrow$ replanning $\rightarrow$ execution) is expressed as an explicit graph pattern rather than hidden callback logic (\textbf{D3}).
At each tick, the runtime materializes the input snapshot and calls \texttt{step()} once. The returned output is then delivered downstream.
Every Flow owns exactly \textit{one} run clock, enforcing single-threaded sequential semantics per module.
Sec.~\ref{sec:case_study} shows how option-style skills and task primitives such as \texttt{OpenDrawer(handle)} become \texttt{SkillCmd} inputs, \texttt{ActionChunk} outputs, and progress/status streams.

\subsection{Synchronization Policies: The Deterministic Bridge}
\label{subsec:framework_sync}
CSFs define causal semantics; Retriever adds the runtime contract for clocked \texttt{step()} modules.
Before each \texttt{step()}, deterministic \texttt{sync} policies turn asynchronous upstream histories into one input snapshot (Fig.~\ref{fig:async-vs-sync-comparison}).
This makes input sampling an explicit timestamped edge contract rather than an implicit queue read (\textbf{D4--D5}).

\begin{figure}[t]
    \centering
    \includegraphics[width=\linewidth]{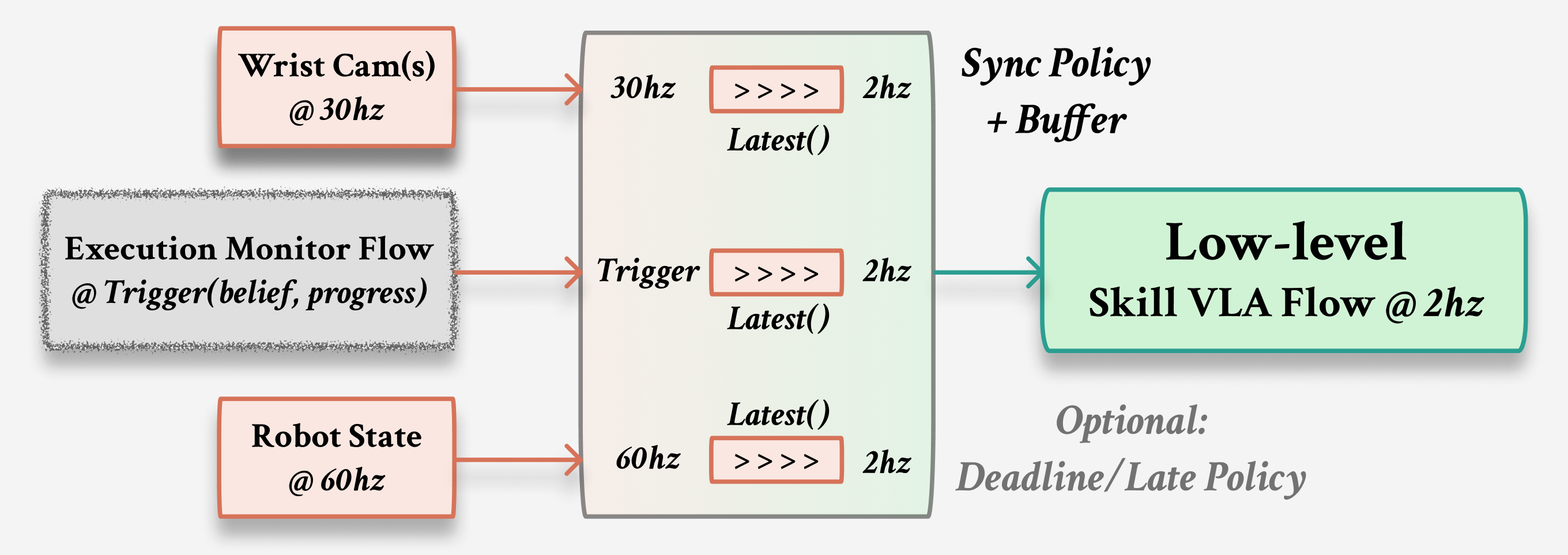}
    \caption{\textbf{Explicit synchronization for asynchronous modules.} A downstream 2Hz VLA \texttt{step()} consumes a synchronized snapshot assembled by deterministic policies such as \texttt{Latest()} from asynchronous streams (30Hz wrist camera, trigger events, 60Hz robot state). Output rates can differ from input clocks.}
    \label{fig:async-vs-sync-comparison}
    \vspace{-1.5em}
\end{figure}
Schematically, each tick follows the same contract:
\begin{lstlisting}[style=retriever, language=Python]
for tick in flow.clock:
    inp = {}
    for edge in flow.inputs:
        hist = edge.history(committed_before=tick)
        inp[edge.port] = sync(edge.policy, hist)
    out = flow.step(inp)
\end{lstlisting}
The returned \texttt{out} becomes visible downstream at the declared commit time.
The key contract is that \texttt{sync} is deterministic over committed events. By default, it reads only events committed before the Flow tick. Same-tick visibility requires a declared deterministic order.
For one incoming edge $u\!\to\!v$, the core operation is \texttt{inp[edge.port]=sync(...)}:
\begin{lstlisting}[style=retriever, language=Python, numbers=none]
visible = edge.history(committed_before=tick)
x_edge  = sync(edge.policy, visible)
\end{lstlisting}
The resulting \texttt{x\_edge} is one field in the ordered input record consumed by \texttt{step()}.\footnote{Formally, \texttt{x\_edge} is the per-edge value $x_{u\to v}(t_k^v)$, and the full snapshot $x_{t_k^v}=\big(x_{u\to v}(t_k^v)\big)_{u\in\mathrm{In}(v)}$ is the $x_{t_k}$ consumed by the local transition $f$ in Sec.~\ref{sec:formulation}. Each edge value may be a frame, held value, short buffer, aggregate, or missing/stale status.}

\noindent\textbf{A concrete \texttt{sync} step.}
Suppose an edge $u\!\to\!v$ carries the timestamped stream (time, value)
\begin{equation*}
    s_{u\to v} = \big[\,(0.10,\,1),\ (0.20,\,4),\ (0.35,\,9),\ (0.60,\,16)\,\big],
\end{equation*}
and $v$ ticks at $t_k^v = 0.50$. The strict prefix visible at this tick drops the future event at $0.60$:
\begin{equation*}
    s_{u\to v}\big|_{<0.50} = \big[\,(0.10,\,1),\ (0.20,\,4),\ (0.35,\,9)\,\big].
\end{equation*}
The \emph{same} history then yields different snapshots under different policies:
\begin{lstlisting}[style=retriever, language=Python, numbers=none]
visible = history.before(0.50)

sync(Latest(),     visible) -> 9
sync(Window(0.30), visible) -> [4, 9]
sync(Window(0.45), visible) -> [1, 4, 9]
\end{lstlisting}
The snapshot is a deterministic function of the committed history, not of queue order. With deterministic Flow code and \texttt{sync} policies, a fixed clock/commit trace therefore determines the network trace.\footnote{For fixed discrete event/synchronization paths and differentiable module dynamics, this also gives well-defined path gradients; Sec.~\ref{sec:eval} uses this property in a differentiable simulation study.} Sec.~\ref{sec:case_study} wires these policies onto the robot pipeline.

%%%%%%%%%%%%%%%%%%%%%%%%%%%%%%%%%%%%%%%%%%%%%%%%%%
%%%%%%%%%%%%%%%%%%%%%%%%%%%%%%%%%%%%%%%%%%%%%%%%%%

Common edge input-materialization policies include \texttt{Latest}, \texttt{Window}, and chunk-style buffering or playback; \texttt{Join} assembles the sampled edge values into the input record consumed by \texttt{step()}.
Table~\ref{tab:operators} summarizes the core computation, synchronization, and record-assembly operators.
\begin{table}[t]
    \centering
    \scriptsize
    \setlength{\tabcolsep}{3pt}
    \renewcommand{\arraystretch}{1.08}
    \begin{tabular}{@{}p{1.45cm}p{1.35cm}p{4.8cm}@{}}
    \toprule
    \textbf{Operator} & \textbf{Role} & \textbf{Contract / Typical Use} \\
    \midrule
    \texttt{Map} & Stateless & Pointwise transformation of one snapshot into another. \\
    \texttt{Scan} & Stateful & Folds input history into local state, e.g., belief or monitor state. \\
    \texttt{Latest} & Sync & Sample-and-hold at the downstream tick, optionally with bounded staleness. \\
    \texttt{Window} & Sync & Returns a bounded temporal buffer, e.g., last $N$ frames. \\
    \texttt{Join} & Bundle & Bundles synchronized edge values into one downstream input record. \\
    \bottomrule
    \end{tabular}
    \caption{\textbf{Core Retriever operators.} \texttt{Map} and \texttt{Scan} define causal computation within a module; \texttt{Latest} and \texttt{Window} sample asynchronous edge histories at module boundaries; \texttt{Join} assembles multi-input records.}
    \label{tab:operators}
    \vspace{-1.2em}
\end{table}
Temporal chunking patterns such as action-chunk playback are built from these primitives as buffering policies or small Flows.
Making the rules explicit replaces ``whatever message happens to be in the queue'' with a deterministic snapshot function of history (Sec.~\ref{sec:formulation}).

\textbf{Lag \& deadline semantics.}
When compute cannot keep up, Retriever makes the fallback explicit (\textbf{D4}).
A policy can hold the last valid value for a bounded duration, mark an input missing or expired, trigger a safe fallback, or raise an error in debug mode.
A blocking contract can instead require a fresh unread event before a step proceeds.
Rate clocks also declare whether lagging ticks are dropped, replayed to catch up, warned, or treated as errors.
Because replay depends on the resulting tick/commit trace, production logs must preserve lag decisions rather than reconstruct them from wall time.
We expand practical contracts and examples in Appendix~\ref{app:framework_ext}.

%%%%%%%%%%%%%%%%%%%%%%%%%%%%%%%%%%%%%%%%%%%%%%%%%%%%%%%%%%%
% (drawer demonstration moved to Fig.~\ref{fig:task-demos} in Sec. 6)
%%%%%%%%%%%%%%%%%%%%%%%%%%%%%%%%%%%%%%%%%%%%%%%%%%%%%%%%%%%

\subsection{Representation Power}
\label{subsec:framework_rep}

Why choose a temporal dataflow graph?
The agent in Fig.~\ref{fig:case-study-pipeline} combines high-rate control, medium-rate skills, and slow belief-space planning~\citep{kaelbling2012unifying,zhao2025seeing}.
We show that temporal dataflow graphs of CSFs are \textbf{expressive} enough for policies that act at discrete decision times, keep bounded memory, and sample asynchronous inputs deterministically.
Such policies need only pointwise maps, stateful scans, deterministic edge sampling, and record assembly (Table~\ref{tab:operators}).
These roles match standard primitives in Functional Reactive Programming (FRP)~\cite{Elliott1997Fran, Wan2000FRP} and stream processing~\cite{Abelson1996SICP}, giving the sufficient basis formalized below.

\begin{thm}[Representability]
\label{thm:representability_main}
Any deterministic policy in the \asyncmodel with locally finite decision times, fixed-size memory, finite aligned inputs from committed observation histories via deterministic \texttt{sync} policies, and strictly causal feedback can be implemented by a Retriever graph with the same action stream for every observation history.
\end{thm}

In programming terms, $\pi$ is a control loop: read inputs, update state, and emit an action.
The proof (Appendix~\ref{app:proof_representation}) makes this concrete.
A single-rate policy becomes a \texttt{Scan} node, optionally followed by a \texttt{Map}, under one \texttt{Clock}; a multi-rate policy is several such loops coupled through deterministic \texttt{sync} edges.
The decision-time and input-sampling conditions are not bookkeeping.
Changing the rate or sampling rule closes a different loop with the environment and therefore defines a different policy.
Clocks and \texttt{sync} policies are part of the policy, not mere scheduling details.
Thus, the Retriever graph model is sufficient for the bounded-memory, multi-rate policies targeted here.

\textbf{Stochasticity.}
Theorem~\ref{thm:representability_main} extends to stochastic policies and environments.
A policy that samples actions---an RL policy or a temperature-sampled VLM---can be written as a deterministic sampler given its input and explicit random bits: $a_k \sim \pi(\cdot \mid z_k)$ becomes $a_k=f(z_k,\epsilon_k)$. One extra stream of realized draws, seeds, or logged model outputs therefore recovers the theorem and supplies the replay contract below (Thm.~\ref{thm:functional_determinism_main}).
Environment randomness needs no extension, since every guarantee is a per-trace statement; Appendix~\ref{app:proof_stochastic} expands both points.

\subsection{Runtime: Graph IR to Actors}
Retriever separates graph definition from execution to support \textbf{D6} without losing \textbf{D4--D5}.
Appendix~\ref{app:framework_ext} provides additional detail.
\begin{enumerate}
    \item \textbf{Definition (Python):} Users define Flows and connect them into a graph. This code is declarative; no threads are spawned yet.
    \item \textbf{Compilation (IR):} The system compiles the Python objects into a \textit{Static Intermediate Representation (IR)}. The IR captures topology, clock definitions, and synchronization parameters.
    \item \textbf{Execution (Actors):} The IR is sent to a backend. Each \textbf{actor} typically runs one Flow with its own state; edges become message channels and deterministic adapters.
\end{enumerate}

\textbf{Backend Mapping.}
Once compiled, the same graph can run with supported executors (\textbf{D6}), including in-process stepping, multiprocessing, and Dora.
Retriever preserves the graph meaning, clocks, and input-consumption contracts.
ROS 2, Dora, and other messaging substrates still handle transport and device integration.

\textbf{Functional Determinism, Logging, and Replay.}
Retriever's replay contract can be stated as the following theorem.
\begin{thm}[Functional Determinism]
\label{thm:functional_determinism_main}
Fix deterministic Flow code and \texttt{sync} policies, initial states, committed external-input histories, and a clock/commit trace (including lag decisions), with strictly causal feedback.
If stochastic calls are replayed from logged seeds or outputs, the output history is unique and independent of runtime scheduling.
\end{thm}

This contract matters because closed-loop nondeterminism compounds: a scheduling race can change one action, then future observations, and eventually the whole trace.
Explicit clocks and \texttt{sync} policies turn functional determinism into a logging and replay interface, following a key property of KPN (Sec.~\ref{sec:related_work}).\footnote{VLM/VLA modules may still be stochastic because of sampling temperature or API-level nondeterminism. Retriever treats such randomness as part of the input stream: deterministic replay requires logging each module's random seed or raw output so replay consumes the logged values rather than re-invoking the model.}
We formalize this as \emph{trace determinism}---the per-trace statement of functional determinism---and prove it in Appendix~\ref{app:proof_determinism}.
This enables deterministic replay and dataset extraction (\textbf{D5}).
% Because, Retriever supports deterministic replay and dataset extraction (\textbf{D5}); Appendix~\ref{app:framework_ext} provides a concrete trace/log view.
Sec.~\ref{sec:eval} studies this empirically in a differentiable hybrid physics example: Retriever fixes the realized discrete event path under replay, so the backward pass differentiates the same graph across runs rather than a schedule-dependent graph.
Determinism here concerns the discrete event path, not bitwise GPU numerics (Appendix~\ref{app:proof_stochastic}).
% \todoP{2}{Add (to appendix?) concrete log schema (event IDs + consumed input IDs) and a short replay example; connect to determinism experiment.}
% This transforms debugging from "chasing race conditions" to "inspecting a deterministic trace,".
% ensures that \textit{imitation learning} training data exactly matches the inputs seen during inference, resolving a common source of distribution shift in robotic learning.

% \textbf{Functional determinism.}
% \begin{itemize}
%     \item \textbf{Reproducibility:} inputs to each \texttt{step()} are uniquely determined by the data log, independent of thread interleaving.
%     \item \textbf{Identifiability:} training pairs $(z_t, a_t)$ are well-defined, enabling valid dataset extraction for learning.
% \end{itemize}

% !TEX root = ../main_arxiv.tex
\section{Case Study: Closed-loop Hierarchical Manipulation with Human-in-the-Loop}
\label{sec:case_study}

We implement a long-horizon tabletop manipulation pipeline that couples partial observability, slow planning, medium-rate skills, and fast control. We use \texttt{Retriever-0} to denote the representative pipeline pattern in Fig.~\ref{fig:teaser}, instantiated at full scale in Fig.~\ref{fig:case-study-pipeline}.

\textit{Case-study notation.}
We denote the main Flows in Fig.~\ref{fig:case-study-pipeline} as $F_{\texttt{cam}}$ (observation), $F_{\texttt{belief}}$ (belief/memory), $F_{\texttt{plan}}$ (planning VLM), $F_{\texttt{exe}}$ (execution monitor), $F_{\texttt{skill}}$ (skill policy), and $F_{\texttt{ctrl}}$ (controller).
We reserve \texttt{PlanChunk}, \texttt{SkillCmd}, \texttt{ActionChunk}, and \texttt{SkillProgress} for typed messages.\footnote{Lower-case ``plan chunking'' and ``action chunking'' refer to the corresponding runtime patterns rather than the message datatypes.}

\textit{Key technical ideas.}
The case study instantiates Retriever around three runtime ideas:
\begin{itemize}[leftmargin=*]
    \item \textbf{Temporal and plan chunking.} We generalize action chunking~\cite{zhao2023learning} to time-extended objects at multiple levels: \texttt{ActionChunk}s bridge VLA inference to high-rate control, while \texttt{PlanChunk}s let slow high-level planning overlap with asynchronous skill execution.
    \item \textbf{Belief-conditioned planning.} The belief Flow emits compact task facts over objects, containers, and task state, supporting short conditional skill programs such as \texttt{IF drawer contains the spice THEN pick it ELSE continue search}~\cite{zhao2025seeing}.
    \item \textbf{Progress-aware switching.} The skill layer includes a progress-prediction module whose output lets the monitor switch skills, trigger fallback, or replan at explicit handoff points.
\end{itemize}

\subsection{Flow Definition, I/O, and Run Clocks}
We implement the pipeline as six major Flows with explicit I/O contracts and run clocks.
Each Flow owns local state, declares when it wakes, consumes an aligned input record, and emits typed outputs.
The runtime handles synchronization and backend data passing between Flows.
\begin{itemize}[leftmargin=*]
    \item \textbf{$F_{\texttt{cam}}$ (Observation Flow)}: emits timestamped camera observations at 30Hz and robot state at 60Hz.
    \item \textbf{$F_{\texttt{belief}}$ (Belief/Memory Flow)}: maintains belief $b$ as local state and updates only after \texttt{inspect(a drawer)}, emitting compact three-valued task predicates such as \texttt{contains(top-left drawer, spice)} $\in \{\texttt{Yes}, \texttt{No}, \texttt{Unknown}\}$.\footnote{This three-valued predicate interface is supported by a belief-space planning theory/algorithm \cite{zhao2025seeing}.}
    \item \textbf{$F_{\texttt{plan}}$ (Planning VLM Flow)}: consumes goal + belief $b$ and produces a bounded-horizon \texttt{PlanChunk}, often a small conditional tree for partial observability. It runs asynchronously when receiving a \texttt{replan} trigger from $F_{\texttt{exe}}$.
    \item \textbf{$F_{\texttt{exe}}$ (Execution Monitor Flow)}: triggered by belief updates and skill-progress predictions; buffers \texttt{PlanChunk}s, emits \texttt{replan} triggers, maintains the active \texttt{SkillCmd}, resolves belief-conditioned branches, and decides skill handoff or replanning.
    \item \textbf{$F_{\texttt{skill}}$ (Skill VLA Flow)}: uses the $\texttt{pi05}$ ($\pi_{0.5}$) model from~\cite{pi05} to execute the active \texttt{SkillCmd} at $\sim$2Hz. It emits an \texttt{ActionChunk} and a progress prediction; inspection skills also emit \texttt{inspection\_done}.
    \item \textbf{$F_{\texttt{ctrl}}$ (Controller Flow)}: consumes \texttt{ActionChunk}s and produces 200Hz motor commands.
\end{itemize}
Appendix~\ref{app:case_study_ext} gives the lower-level belief representation, \texttt{PlanChunk} structure, and skill library details.

\textit{Programming view of a skill Flow.}
We use ``skill'' operationally: a finite-memory behavior that is started by a \texttt{SkillCmd}, reads synchronized observations and robot state, and returns low-level actions plus status.

\begin{defn}[Skill]
A \emph{skill} is a finite-memory, temporally extended behavior with three parts:
(i)~\emph{initiation}: a trigger condition that starts it;
(ii)~\emph{execution}: local state and a within-skill transition that emits low-level commands or action chunks while active;
    (iii)~\emph{termination}: conditions that end it, emitting status, success, failure, or handoff events.
\end{defn}

This is the initiation--policy--termination structure of options and skill formalisms~\cite{Sutton1999Options,Konidaris2018Skills}, augmented with status and handoff events for asynchronous monitoring.

\begin{prop}[Skills as Flows]
Every deterministic finite-memory skill whose initiation and termination read timestamped input histories is representable as a Retriever Flow.
When a separate monitor manages initiation and handoff, as in our pipeline, the skill splits into a monitor Flow and a skill Flow linked by \texttt{SkillCmd}, progress prediction, and terminal-status streams.
\end{prop}

\noindent This is an instance of Theorem~\ref{thm:representability_main}, with each part of the definition mapped to the Flow interface.
Initiation~(i) arrives through synchronized trigger events or \texttt{SkillCmd}s.
Execution~(ii) stores the active flag and skill memory in Flow state and runs the within-skill policy in \texttt{step()}.
Termination~(iii) emits status and terminal events as typed output streams.
The VLA skill Flow therefore has the following pseudocode contract:
\begin{minipage}{\columnwidth}
\begin{lstlisting}[style=retriever, language=Python, numbers=none, basicstyle=\ttfamily\scriptsize, xleftmargin=4pt, framexleftmargin=4pt, aboveskip=1.0em, belowskip=1.0em]
F_skill.clock = Rate(hz=2)
F_skill.inputs = {
    obs: Latest(F_cam.obs),
    state: Latest(F_cam.state),
    cmd: Latest(F_exe.SkillCmd),
}

def F_skill.step(obs, state, cmd, memory):
    action_chunk = VLA(obs, state, cmd, memory)
    progress = progress_prediction_module(obs, state, cmd, memory)
    return ActionChunk(action_chunk), SkillProgress(progress)
\end{lstlisting}
\end{minipage}
Symbolic operators, option-like skills, and learned VLA primitives can use this interface once their commands, observations, state, actions, progress predictions, and terminal status are exposed as typed streams.\footnote{Retriever does not replace option or TAMP formalisms; it provides the runtime substrate that composes them with asynchronous perception and control.}

\subsection{Connections and Synchronization}
Figure~\ref{fig:case-study-pipeline} separates \emph{main data connections} (gray arrows), \emph{control connections} (red triggers), and \emph{auxiliary inputs} (blue).

\paragraph{Main data connections (gray).}
\begin{itemize}[leftmargin=*]
    \item \textbf{Perception to belief and skill.} The 30Hz observation stream from $F_{\texttt{cam}}$ is synchronized to the ticks of $F_{\texttt{belief}}$ and $F_{\texttt{skill}}$ via \texttt{Latest} (sample-and-hold with a bounded-staleness check). This produces a well-defined input snapshot for each belief update and VLA \texttt{step()} (Fig.~\ref{fig:async-vs-sync-comparison}).
    \item \textbf{Skill to control.} $F_{\texttt{skill}}$ emits an \texttt{ActionChunk}; $F_{\texttt{ctrl}}$ consumes it through deterministic chunk playback, sampling one action per 200Hz control tick. This separates the slow VLA inference clock from the safety-critical control clock.
\end{itemize}

\begin{figure*}[t]
    \centering
    % \includesvg[width=\linewidth]{figures/illustration/temporal_chunking.svg}
    \includegraphics[width=0.97\linewidth]{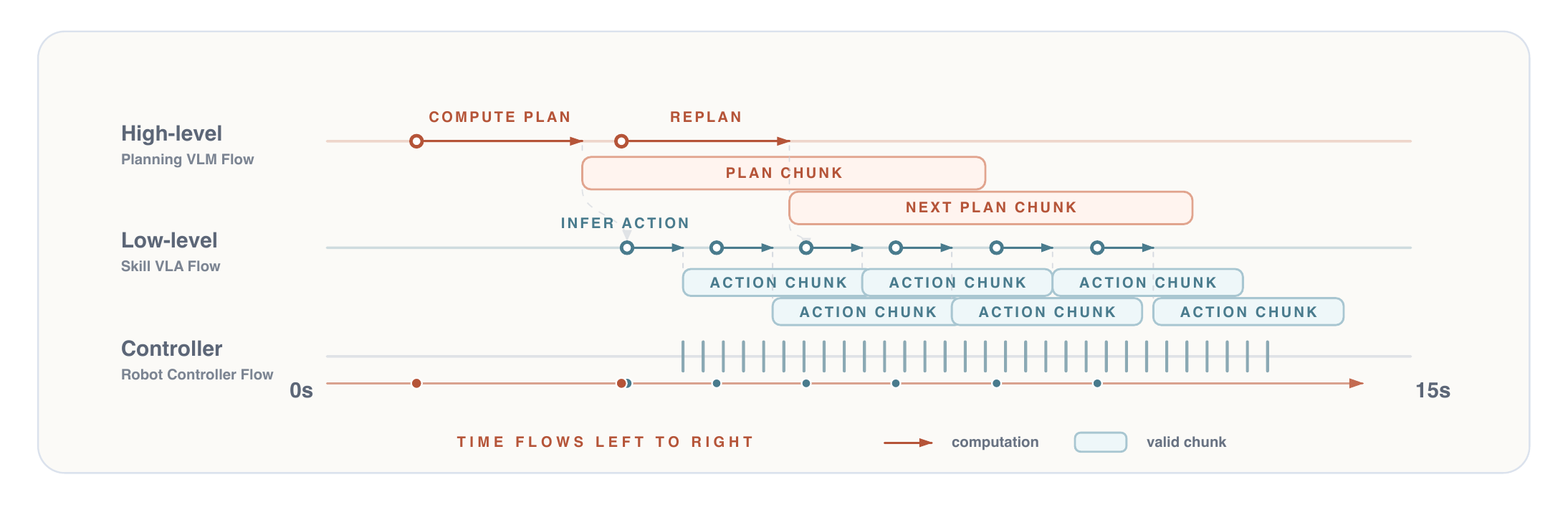}
    \caption{\textbf{Temporal chunking} (\textbf{D2}, \textbf{D4}). Time-extended outputs appear at two levels: the controller plays buffered \texttt{ActionChunk}s, while $F_{\texttt{exe}}$ commits \texttt{PlanChunk}s at skill boundaries. Arrows are computation intervals; blocks are chunk-validity horizons. Overlap keeps downstream clocks running while the next plan or action chunk is produced.}
    \label{fig:temporal_chunking}
    \vspace{-0.8em}
\end{figure*}

\paragraph{Control connections (red): plan chunking and monitored handoff.}
The central coordinator is $F_{\texttt{exe}}$.
It receives belief updates from $F_{\texttt{belief}}$, \texttt{PlanChunk} proposals from $F_{\texttt{plan}}$, and progress predictions from $F_{\texttt{skill}}$.
It then selects the next active \texttt{SkillCmd}.\footnote{The $F_{\texttt{exe}}\!\leftrightarrow\!F_{\texttt{skill}}$ cycle remains strictly causal: progress emitted at a skill tick is consumed only at a later monitor tick.}
\begin{itemize}[leftmargin=*]
    \item \textbf{Trigger chain.} Information-gathering skills emit \texttt{inspection\_done}, and belief updates emit \texttt{belief\_updated}. The monitor then requests \texttt{replan}, receives a \texttt{PlanChunk}, and commits future skills only at explicit handoff boundaries.
    \item \textbf{Plan chunking.} $F_{\texttt{plan}}$ proposes a short-horizon \texttt{PlanChunk} while the current skill is still running. $F_{\texttt{exe}}$ buffers this proposal and switches to the next \texttt{SkillCmd} only at monitor-approved handoff points, so planning overlaps execution without mid-skill rewrites. The default commit rule never interrupts the active skill: if the agent is executing skill \texttt{C} when a new plan \texttt{A,B,C,D,E} arrives, it finishes \texttt{C} and continues with \texttt{D,E}.
    \item \textbf{Stale-plan handling.} Each proposal records the belief snapshot and current-skill context that produced it. The monitor commits it only while that context remains compatible with the active execution state; otherwise it drops the proposal or replans.
    \item \textbf{Progress-aware handoff.} The monitor uses progress predictions to approve handoff, detect stalls, or trigger fallback.
\end{itemize}
Together with action-chunk playback on the data path, these monitor policies instantiate temporal chunking (Fig.~\ref{fig:temporal_chunking}) at two levels: plans are chunks over future skills, while actions are chunks over future controller commands.
As with action chunking~\cite{zhao2023learning}, each commit rule is one strategy within this pattern, and it can be tuned or replaced independently of the rest of the pipeline.

\paragraph{Auxiliary inputs (blue).}
A \texttt{LanguageGoal} input starts or restarts planning, while an optional \texttt{Teleoperation} input can override the action stream.
Both are source Flows integrated by $F_{\texttt{exe}}$ and logged for replay.
Because they have no upstream trigger events, they run on periodic clocks (e.g., 1Hz goal polling, 100Hz teleoperation) and emit events that trigger downstream Flows.
Appendix~\ref{app:case_study_ext} gives the contract-level task interfaces.

\section{Empirical Evaluation}
\label{sec:eval}

%%%%%%%%%%%%%%%%%%%%%%%%%%%%%%%%%%%%%%%%%%%%%%%%%%%%%%%%%
\begin{figure}[t]
    \centering
    \begin{minipage}{0.49\linewidth}
        \centering
        \includegraphics[width=\linewidth]{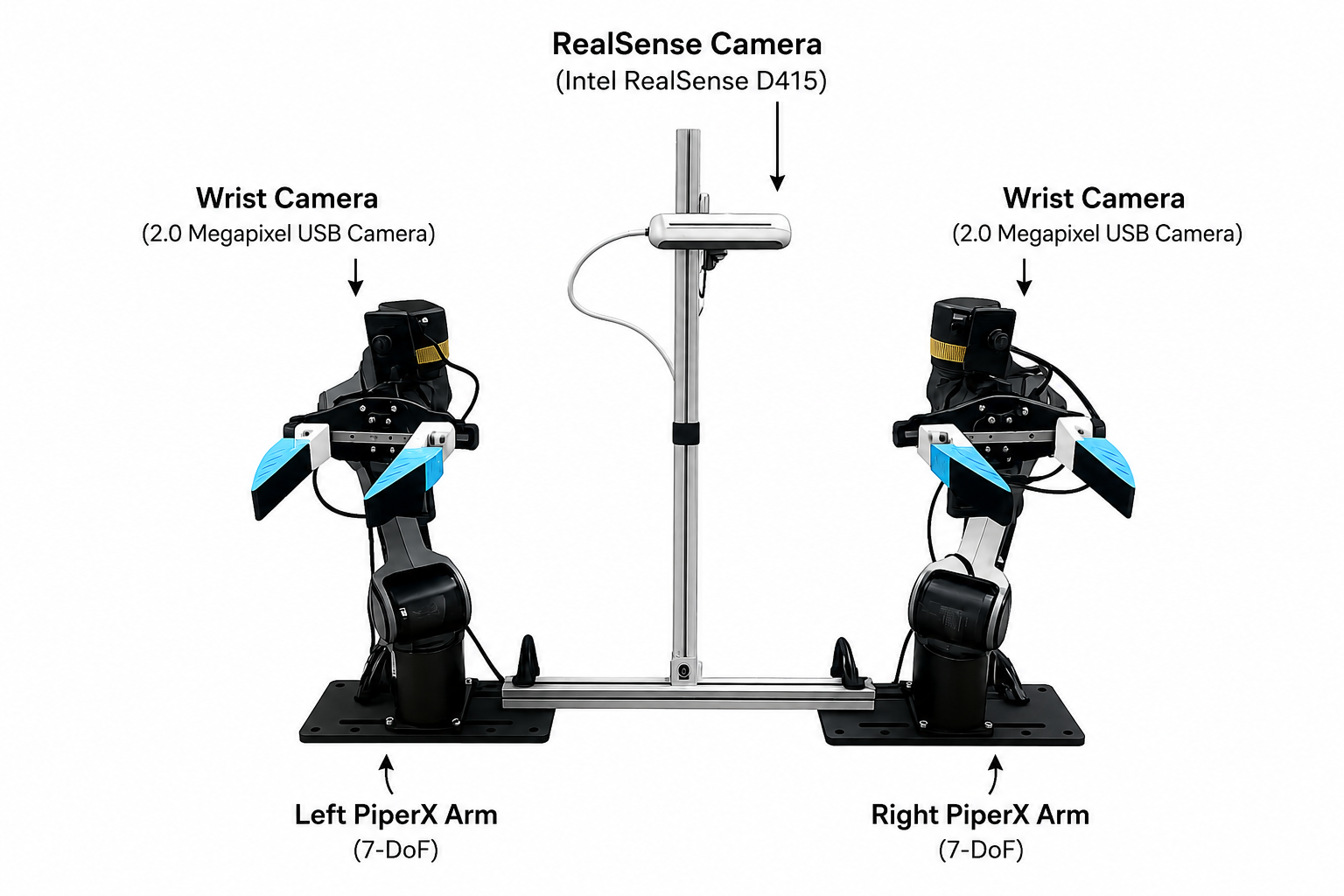}\\[0.15em]
        {\footnotesize (a) Bimanual platform}
    \end{minipage}\hfill
    \begin{minipage}{0.49\linewidth}
        \centering
        \includegraphics[width=\linewidth]{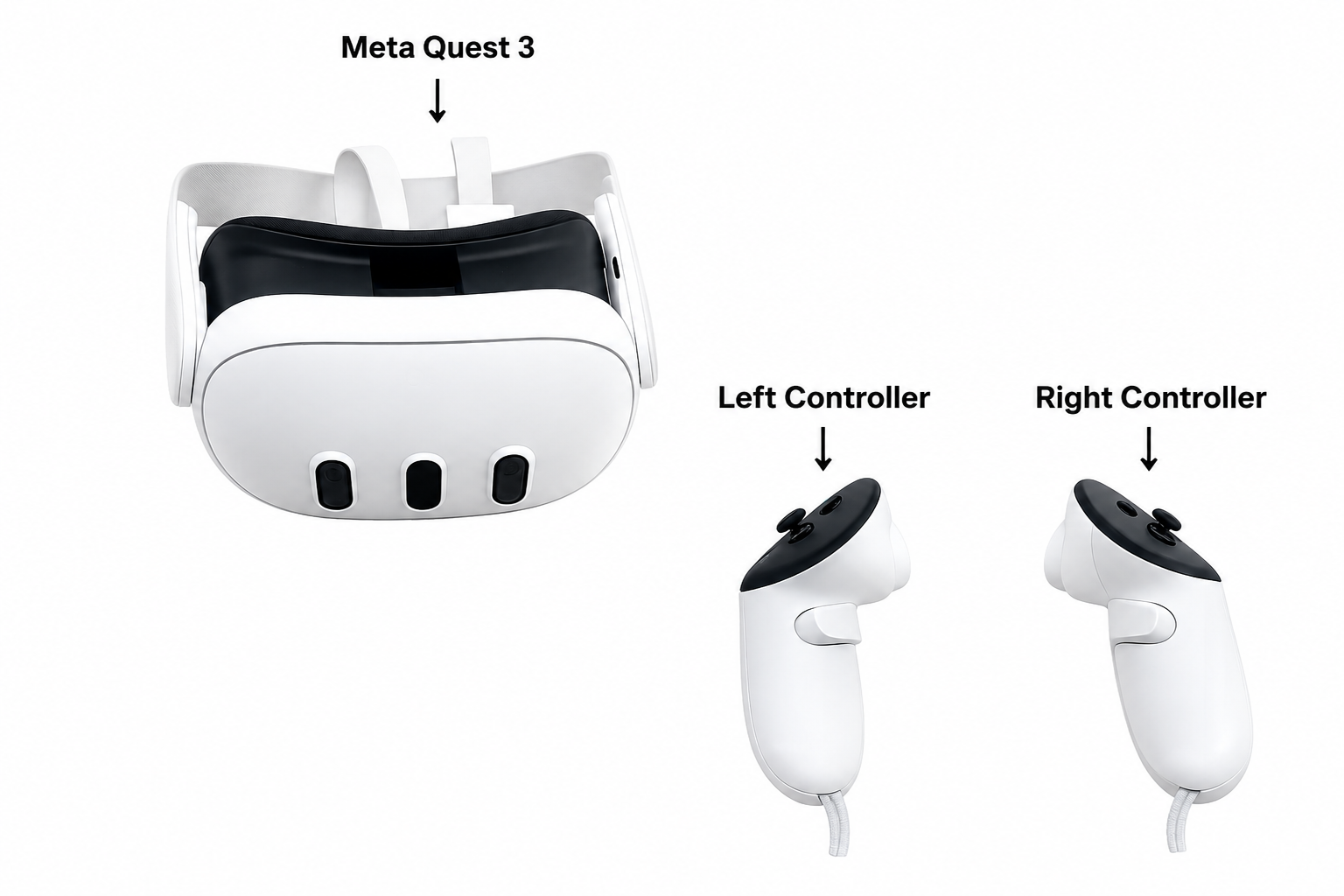}\\[0.15em]
        {\footnotesize (b) Meta Quest~3 teleoperation}
    \end{minipage}
    \vspace{-0.4em}
    \caption{\textbf{Robot hardware used in the real-robot evaluation.} \textbf{(a)}~The bimanual PiperX platform uses wrist cameras and an external RealSense camera. \textbf{(b)}~Optional Meta Quest~3 teleoperation provides human-in-the-loop recovery and data collection, integrated as an optional input Flow and logged for replay.}
    \label{fig:robot_hardware}
    \vspace{-0.5em}
\end{figure}
%%%%%%%%%%%%%%%%%%%%%%%%%%%%%%%%%%%%%%%%%%%%%%%%%%%%%%%%%

%%%%%%%%%%%%%%%%%%%%%%%%%%%%%%%%%%%%%%%%%%%%%%%%%%%%%%%%%
\begin{table}[t]
    \centering
    \scriptsize
    \setlength{\tabcolsep}{3.2pt}
    \begin{tabular}{lcccc}
    \toprule
    \textbf{Condition} & \multicolumn{2}{c}{\textbf{Spice search and seasoning}} & \multicolumn{2}{c}{\textbf{Bag retrieval}} \\
    \cmidrule(lr){2-3}\cmidrule(lr){4-5}
     & Proj. time & Task progress & Proj. time & Task progress \\
    \midrule
    \texttt{Retriever-0} (full) & 56s & 73\% & 36s & 100\% \\
    \textminus~Plan chunking & 65s & 63\% & 37s & 93\% \\
    \textminus~Progress prediction & -- & 12.5\% & -- & 25\% \\
    \textminus~Belief memory & -- & 20\% & 45s & 66\% \\
    \textminus~Replanning & -- & 15\% & -- & 47\% \\
    % \textminus~Async execution & -- & -- & -- & -- \\
    +~Human-in-the-loop & 50s & 100\% & 35s & 100\% \\
    \midrule
    $\pi_{0.5}$ (VLA only) & -- & 0\% & 41s & 41\% \\
    Diffusion policy only & -- & 0\% & 40s & 23\% \\
    \bottomrule
    \end{tabular}
    \caption{\textit{Real-robot evaluation and ablations.} We report task progress $p$ and projected total time $\widehat{T}=T_{\mathrm{best}}/p$, rounded to the nearest second. Here $T_{\mathrm{best}}$ is the elapsed time to the best achieved progress. This normalization makes partial runs comparable; it is not a measured completion time. ``--'' marks values that are not meaningful, such as stalled or non-terminating runs. \texttt{Retriever-0} is the pipeline in Fig.~\ref{fig:teaser}; Appendix~\ref{app:case_study_ext} defines progress scoring.}
    \label{tab:robot_eval}
    \vspace{-1.5em}
\end{table}
%%%%%%%%%%%%%%%%%%%%%%%%%%%%%%%%%%%%%%%%%%%%%%%%%%%%%%%%%

Our evaluation asks three questions:
(1) \textbf{Capability}: Can Retriever support long-horizon real-world manipulation that coordinates slow reasoning (VLM), medium-frequency skills (VLA), and fast control?
(2) \textbf{Efficiency}: Does asynchronous execution reduce end-to-end task time compared to blocking implementations?
(3) \textbf{Usability}: Does making time explicit make agent behavior easier to reproduce, replay, and debug?

%%%%%%%%%%%%%%%%%%%%%%%%%%%%%%%%%%%%%%%%%%%%%%%%%%%%%%%%%
\begin{figure*}[t]
    \centering
    \includegraphics[width=0.97\linewidth]{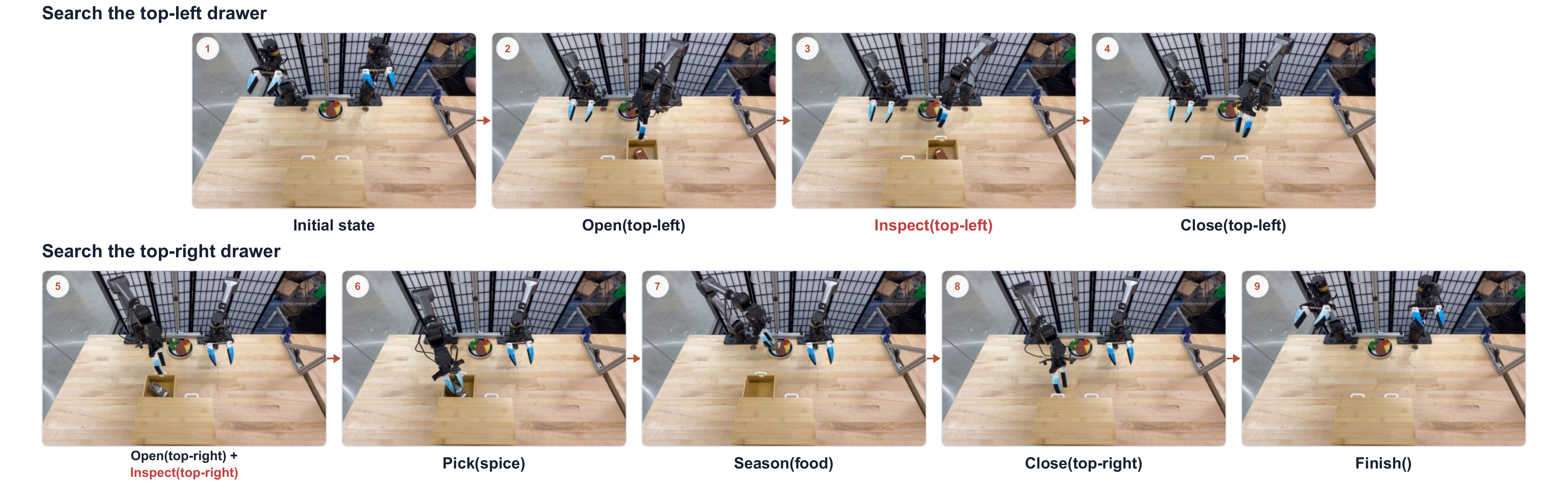}\\[0.9em]
    \includegraphics[width=0.97\linewidth]{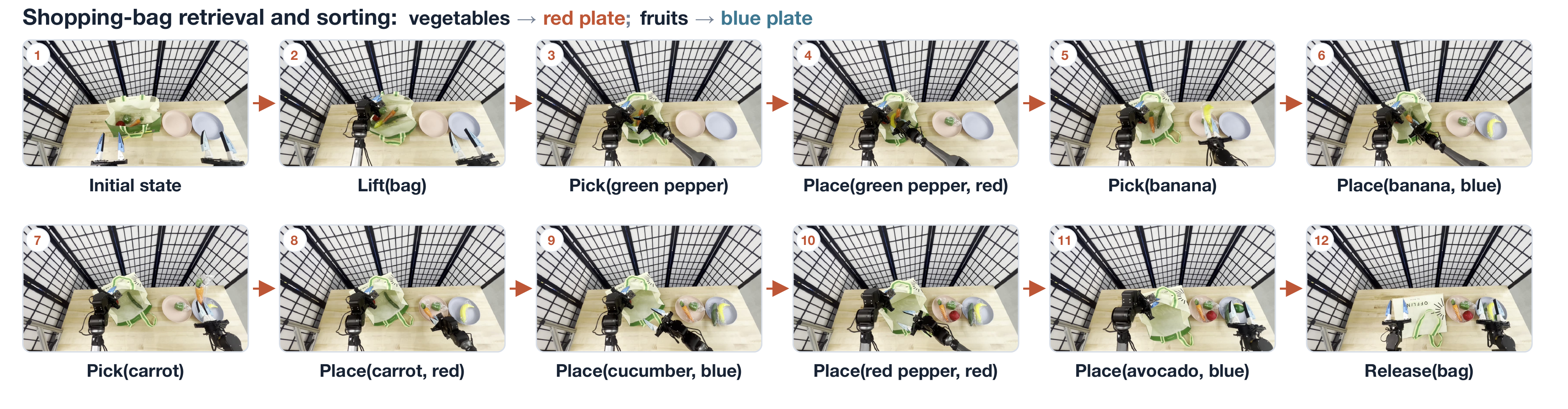}
    \caption{\textbf{Real-robot task demonstrations.} \emph{Top (spice search and seasoning)}: a representative closed-loop run first inspects the top-left drawer and updates belief after not finding the target spice, then searches the top-right drawer, picks the spice, seasons the food, puts it back, closes the drawer, and terminates. Numbered frames show the executed skill sequence; information-gathering actions (red) trigger belief updates, where the plan branches between continuing the search and executing the seasoning sequence. \emph{Bottom (bag retrieval and sorting)}: in one closed-loop run, the robot lifts a deformable shopping bag, retrieves several items, sorts vegetables onto the red plate and fruits onto the blue plate, and finally releases the bag.}
    \label{fig:task-demos}
    \vspace{-1.0em}
\end{figure*}
%%%%%%%%%%%%%%%%%%%%%%%%%%%%%%%%%%%%%%%%%%%%%%%%%%%%%%%%%

\subsection{Capability: Closed-Loop Asynchronous Pipelines}
We first test whether Retriever supports multi-rate systems that are hard to build with existing tools.
Both tasks close the loop across every module---perception, belief memory, VLM planning, VLA skills, and the execution monitor: spice search exercises the belief-update and replanning cycle on every drawer, while bag retrieval stresses the tight perception--manipulation loop under deformation.
We deploy the asynchronous pipeline from Sec.~\ref{sec:case_study} on a bimanual robot platform (Fig.~\ref{fig:robot_hardware}) in two tasks; Fig.~\ref{fig:task-demos} shows representative runs.

\textit{(1) Bag retrieval:} A deformable shopping bag holding several objects sits beside two plates. One arm lifts and steadies the bag while the other reaches in to retrieve the target item; in the longer sorting variant (Fig.~\ref{fig:task-demos}, bottom), the robot retrieves every item and places each on the plate matching its category before releasing the bag. Because the bag deforms with every extraction, VLM scene understanding and VLA manipulation must remain in a tight loop.

\textit{(2) Spice search and seasoning:} Several spice bottles are placed randomly across four closed drawers, and the language goal names the one to use (e.g., ``season the steak with black pepper''). The robot opens a drawer by its handle, inspects the container, and updates its belief: if the target spice is not there, it closes the drawer and continues the search; once found, it picks the bottle, seasons the food, puts the bottle back, and closes the drawer. Locating the right bottle takes multiple VLM calls, while the pick-and-place skills run under high-frequency control.

Both tasks are scored by \emph{task progress}.
With the full pipeline, \texttt{Retriever-0} reaches 73\% task progress on spice search (averaged over four drawers) and 100\% on bag retrieval (Table~\ref{tab:robot_eval}). Based on qualitative inspection of these runs, low-level policy execution appeared to be the main bottleneck; high-level reasoning and skill switching were comparatively reliable. The human-in-the-loop version uses teleoperation to collect data or correct the policy, which achieves full success.

We train the shared low-level skill layer with roughly 200--300 complete expert episodes across both tasks, each about 1--3 minutes long, for a total of a few hours of robot data.
In this setting, either fine-tuning $\pi_{0.5}$ or training a multi-task diffusion policy from scratch works.
The optional \texttt{Teleoperation} Flow uses a Meta Quest~3 controller interface (Fig.~\ref{fig:robot_hardware}) for recovery and data collection in the human-in-the-loop condition.
Its actions are logged alongside the robot-agent event trace.
Task protocols, skill catalogs, and progress scoring definitions are summarized in Appendix~\ref{app:case_study_ext}; hardware/setup and benchmark details are in Appendix~\ref{app:eval_ext}.

% We successfully deployed both pipelines. By leveraging Retriever's synchronization primitives, we could compose these disparate modules (event-triggered VLM, 2Hz VLA, 200Hz Controller, 30Hz camera) into a single, cohesive runtime graph. The system executed the long-horizon tasks smoothly, with the synchronization policies automatically handling the hand-offs between the slow planner and the fast controller.

% \todo{Link to supplementary videos showing these tasks.}
% Placeholder Data:
% Confirmed valid execution on real hardware for both tasks.
% Video proofs to be added.

% Task-level results and ablations are summarized in Table~\ref{tab:robot_eval}.

%%%%%%%%%%%%%%%%%%%%%%%%%%%%%%%%%%%%%%%%%%%%%%%%%%%%%%%%%
\begin{figure*}[t]
    \centering
    \begin{minipage}[t]{0.32\linewidth}
        \centering
        \includegraphics[width=\linewidth]{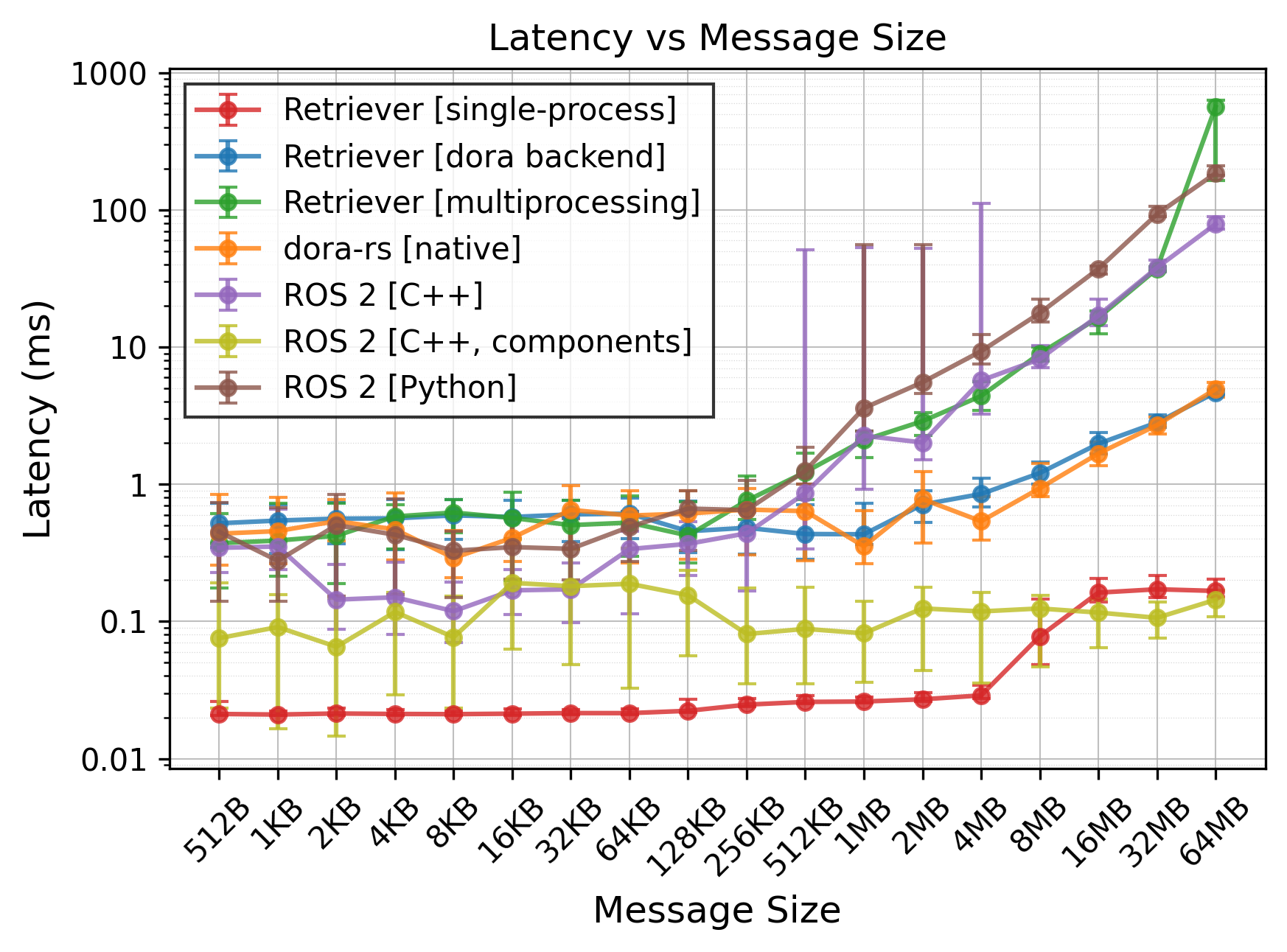}
        \vspace{-0.5em}
    \end{minipage}\hfill
    \begin{minipage}[t]{0.30\linewidth}
        \centering
        \includegraphics[width=\linewidth]{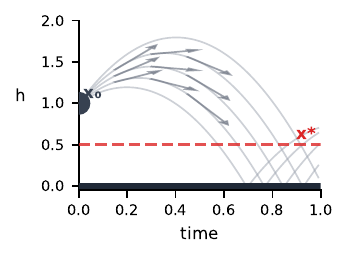}
        \vspace{-0.5em}
    \end{minipage}\hfill
    \begin{minipage}[t]{0.32\linewidth}
        \centering
        \includegraphics[width=\linewidth]{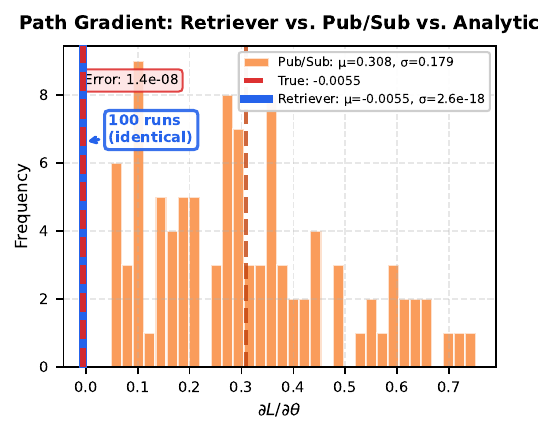}
        \vspace{-0.5em}
    \end{minipage}
    \vspace{-1.0em}
    \caption{\textbf{Efficiency \& Determinism Visualizations.} (a) End-to-end message latency vs.\ payload size (log--log). \texttt{Retriever} with a Dora transport backend closely matches native \texttt{dora-rs}, suggesting low abstraction overhead compared to transport costs. (b) Differentiable physics visualization: ball bouncing with hybrid dynamics. (c) Functional determinism for gradient-based learning: event-time semantics yield identical \emph{path gradients} across runs, while arrival-time pub/sub yields a broad gradient distribution.}
    \vspace{-1.5em}
    \label{fig:results_benchmark}
\end{figure*}
%%%%%%%%%%%%%%%%%%%%%%%%%%%%%%%%%%%%%%%%%%%%%%%%%%%%%%%%%

\subsection{Efficiency: Completion Time \& Ablations}
We next measure message-passing latency, report progress-normalized projected time for the robot tasks, and ablate four architectural features: \textit{plan chunking}, progress prediction, belief memory, and replanning.
Projected time normalizes elapsed time by achieved progress, keeping partial runs comparable when conditions fail at different milestones.
In Table~\ref{tab:robot_eval}, ``--'' marks values without a meaningful measurement, such as stalled or non-terminating runs.

\textbf{Message Passing Latency.} 
We measure Retriever's runtime overhead with a message-passing latency benchmark (Fig.~\ref{fig:results_benchmark}(a)).
Retriever with the Dora backend closely matches native Dora and is faster than ROS 2 process-based Python/C++ baselines.
The figure also includes a stronger ROS 2 C++ \emph{components} baseline, in which composable nodes share one process~\cite{Macenski2023ROS2Composition}.\footnote{This baseline should be compared with Retriever's single-process backend, which has lower latency for small messages used in control and coordination.}
This producer--consumer benchmark isolates abstraction overhead rather than full robot throughput.\footnote{Robot runs also include model inference, device drivers, synchronization policies, and control-loop deadlines.}
Appendix~\ref{app:eval_ext} gives benchmark details.

\textbf{Plan Chunking.} We compare asynchronous multi-step chunking with a synchronous blocking baseline (Table~\ref{tab:robot_eval}).
In the default setting, the VLM plans the next chunk while the VLA executes the current one.
Each \texttt{PlanChunk} contains up to five skills, allowing planning to overlap multiple skill executions.
The ablation plans only one next skill.
The monitor advances to the next \texttt{SkillCmd} only at explicit handoff points indicated by progress prediction.
When no new information is gathered, the one-skill ablation still waits for planning at every skill boundary, adding idle time and increasing projected total time.
% In contrast, Retriever's pipelining completely hides the VLM's inference latency, reducing the total task time by XX.
Plan chunking and handoff semantics are summarized in Appendix~\ref{app:case_study_ext}.
\textbf{Progress prediction.} The ablation replaces the learned progress predictor with VLM queries that ask whether the current skill has finished (Table~\ref{tab:robot_eval}).
The robot approximately completes the first skill, but the VLM cannot reliably detect completion, so the monitor receives no usable skill-switching signal and execution does not advance.
We report the corresponding approximate first-stage progress but do not extrapolate a projected total time from a run that stalls at its first handoff.
See Appendix~\ref{app:case_study_ext} for the interface contract.
\textbf{Belief memory.} Removing belief memory leaves the agent unaware of which drawers it has already searched: on spice search it repeatedly re-searches the first drawer and never terminates, so its projected time is not meaningful. On bag retrieval it finishes but sometimes misses objects.
\textbf{Replanning.} Planning only once decouples the plan from belief updates: after opening a drawer, the plan no longer matches what the agent now knows, so it does not know what to do next. On bag retrieval, the remaining bag contents stay unknown and failures are not recovered. Both runs stall partway, so projected times are omitted.
\textbf{Policy-only baselines.} The end-to-end $\pi_{0.5}$ and from-scratch multi-task diffusion baselines receive camera observations, robot state, and the language goal, but no explicit planner, belief memory, execution monitor, or skill-program structure.
On the drawer task they score $0\%$: all drawers start closed, so a policy alone does not even know where to start without the search-and-inspect behavior that planning and belief provide.
Bag retrieval still requires online perception and target localization, but less explicit belief-space search; the two baselines reach $41\%$ and $23\%$, respectively.
% This dynamic transition reduced idle time by an additional 15\%, further demonstrating that fine-grained, data-driven coordination outperforms static scheduling.
% Placeholder Data:
% Enabling "online chunking" (VLM predicts next 3 steps) reduced avg.\ per-step latency by XXX compared to "no chunking" (querying VLM at every step), as the VLA could execute the current chunk while the VLM prepared the next.

\subsection{Usability: Determinism, Replay, and Debugging}
Finally, we test whether explicit clocks and synchronization make timing-sensitive runs reproducible and inspectable.

\textit{Functional Determinism.} For a differentiable bouncing-ball system, we estimate the gradient of location at $t=1$ with respect to initial velocity $v_0$ (Fig.~\ref{fig:results_benchmark}(b)).
In hybrid systems, the computation graph and its gradient depend on the exact timing of discrete events such as impact.
For fixed input histories, Retriever reproduces the same computation graph across runs and therefore the same path gradients.
Across 100 runs, Retriever yields identical execution traces and a collapsed gradient distribution (std $\approx 10^{-18}$).
By contrast, pub/sub-style arrival semantics produce diverse traces and gradients under mocked scheduling jitter (Fig.~\ref{fig:results_benchmark}(c)); see Appendix~\ref{app:eval_ext}.
% We quantified this by recording a 60-second multi-sensor log and replaying it $N=20$ times. The ROS~2-style baseline (using callbacks) exhibited a mean trace discrepancy of 8.4\% (due to non-deterministic message arrival order). In contrast, Retriever achieved \textbf{0.0\% discrepancy}, producing bit-identical execution traces. This \textbf{functional determinism} allows developers to debug rare timing bugs by simply replaying logs, transforming ``heisenbugs'' into reproducible events.

% Placeholder Data:
% Baseline (ROS~2-style): Mean trace discrepancy of D=8.4% (Levenshtein distance).
% Retriever: Discrepancy of exactly D=0.0%.

% Gradient histogram now included in Fig.~\ref{fig:results_benchmark}.

\textit{Implementation boilerplate.}
We also count user-authored code for the message-latency benchmark: 117 lines in one Retriever file versus 191 lines across four ROS 2 Python files, including build configuration.
This excludes framework internals and is only a coarse measure of benchmark boilerplate.
Appendix~\ref{app:eval_ext} gives the raw comparison.

% !TEX root = ../main_arxiv.tex
\section{Conclusion}
\label{sec:conclusion}

We introduced Retriever, a programming model and runtime for asynchronous, compositional robot-agent systems in open-world manipulation. By making time and synchronization explicit, Retriever lets developers specify when modules run, which inputs they consume, and what must be recorded for replay. Our experiments show Retriever coordinating VLM planning, skill execution, belief updates, and progress monitoring for long-horizon robot manipulation. The same explicit timing makes runs reproducible and easier to debug.

With a fixed replay trace, Retriever's deterministic dataflow contract makes event traces and path gradients reproducible (Thm.~\ref{thm:functional_determinism_main}; Sec.~\ref{sec:eval}). Each recorded step pairs an output with the exact inputs and timing that produced it. These traces may support training future, more integrated end-to-end systems.

\paragraph{Code and materials.}
Retriever is open source: \texttt{pip install retriever-core} (imported as \texttt{retriever}). Source code, documentation, and runnable examples are available at \url{https://github.com/openretriever/retriever} and \url{https://retriever.build}. Applied robot examples are available in GoldenRetriever (\url{https://golden.retriever.build}), with the project home at \url{https://openretriever.org}. Retriever Hub packs let projects load Flows, typed messages, and pipelines by name. Each pack resolves to a pinned commit and is fetched without cloning.

% Retriever makes time and input-consumption semantics explicit for hierarchical robot agents by combining (i) a stream-based formulation and causal stream functions, (ii) a programming model (Flows with explicit run clocks and synchronization policies), and (iii) a runtime that supports deterministic stepping and replay.
% We believe this provides a reusable substrate for building robust closed-loop agents under real-world asynchrony and sets the stage for building community and collecting closed-loop asynchronous agent data.

% \section*{Acknowledgments}
% TODO later

%% Use plainnat to work nicely with natbib.

\newpage
\bibliographystyle{unsrtnat}
\bibliography{references}

\clearpage
% \appendix
\onecolumn
\appendices

\maybeTodoList

\begin{adjustwidth}{0.75in}{0.75in}
% !TEX root = ../main_arxiv.tex
\section{Historical Lineage: Programming Models and System Abstractions for Hardware and Robots}
\label{app:history}

In systems practice, ``independence from hardware'' (including robot hardware) is not a single property.
It typically decomposes into four separable goals:

\begin{enumerate}
    \item \textbf{Portability of meaning (semantic portability).} A program has the same externally observable behavior across machines (modulo performance). This is the hardest kind of independence; it requires explicit semantic contracts.
    \item \textbf{Portability of interfaces (API/ABI portability).} Code compiles and runs across platforms because it targets a stable interface boundary (e.g., POSIX~\cite{POSIX1988}).
    \item \textbf{Portability of placement (location transparency).} Computation can be placed on different cores/machines without rewriting the program; the runtime chooses placement.
    \item \textbf{Portability of performance (performance portability).} The same program attains reasonable efficiency across heterogeneous devices (CPU/GPU/accelerators), usually via compiler/runtime optimization and hardware abstraction layers.
\end{enumerate}

Major computing shifts, from minicomputers to clouds and GPUs, were often absorbed through a \textit{stable systems abstraction}: a small interface shared by many programs and implemented by many backends.
UNIX + C is one canonical example of such a mechanism at the OS + language boundary~\cite{Ritchie1974UNIX, Ritchie1984Evolution, Ritchie1993C}.

Retriever targets a narrow waist for \textit{robotic agent computation graphs}.
The program fixes the wiring, message boundaries, and timing semantics, while supported runtimes and backends can vary without rewriting it.

This appendix traces the systems ideas behind that separation and maps them to Retriever's design choices.

\paragraph{Organization.}
We group the lineage by recurring design problems and the corresponding historical anchors:
\begin{itemize}[leftmargin=*]
    \item \textbf{Uniform interfaces (hardware independence):} UNIX/POSIX and the C portability story~\cite{Ritchie1974UNIX, POSIX1988, Ritchie1984Evolution, Ritchie1993C}.
    \item \textbf{Determinism under concurrency:} KPN and synchronous dataflow (schedule-independent meaning + strict causality)~\cite{Kahn1974Semantics, Lee1987SDF}.
    \item \textbf{Streams over time:} FRP and Arrowized FRP (signals, causality, feedback)~\cite{Nilsson2002FRP, Hughes2000Arrows}.
    \item \textbf{Scalable isolation and placement:} actors and message passing (location transparency without shared-state races)~\cite{Hewitt1973Actor, Armstrong2003Erlang}.
    \item \textbf{Auditability and replay:} event logs as the system of record~\cite{Kreps2011Kafka}.
    \item \textbf{Graph IRs as a stable contract:} modern learning frameworks (graph/IR enables portability + optimization)~\cite{Abadi2016TensorFlow, Paszke2019PyTorch}.
    \item \textbf{Robotics integration vs semantics:} ROS-style middleware provides transport and ecosystem, but not a semantic runtime~\cite{Quigley2009ROS, Macenski2022ROS2}.
    \item \textbf{Decision-theoretic timing:} a brief continuous-time / (PO-)SMDP view of asynchronous decision epochs is summarized in Appendix~\ref{app:formulation_related}.
\end{itemize}
The remainder of this section elaborates each theme and distills the implication for a Retriever-style semantic contract.

\subsection{UNIX: Uniform Interfaces as an Anti-Hardware Strategy}

\subsubsection{The Core Bet: A Small Set of Composable Primitives}
The 1974 UNIX paper frames UNIX as a general-purpose time-sharing OS whose design emphasizes a small set of core mechanisms and a programming environment conducive to composition~\cite{Ritchie1974UNIX}.
Two ideas became especially influential for hardware independence:
\begin{itemize}
    \item \textbf{A uniform I/O interface} spanning files, devices, and inter-process communication.
    \item \textbf{A process model} that makes programs composable through standard streams and pipes.
\end{itemize}
The paper explicitly highlights ``compatible file, device, and inter-process I/O'' as a central property of the system.
This phrase matters because it encodes a portability strategy: if \textit{devices} and \textit{IPC} can be used through the same I/O interface as files, then most user programs are insulated from device-specific details.
Instead of every program embedding device knowledge, device knowledge is concentrated in kernel drivers and the file system namespace.

\subsubsection{Special Files and Pipes}
UNIX describes \textbf{special files} (device nodes) as entries in the file system whose reads/writes are redirected into device-specific routines~\cite{Ritchie1974UNIX}.
This is the ``universal I/O'' idea in operational form: programs depend on \texttt{open/read/write/close}; devices implement those operations behind the boundary.

For Retriever, user code depends on a small set of \textit{graph-level} operations, while transport, memory movement, and clocking details stay behind the runtime boundary.

Ritchie's retrospective emphasizes pipelined commands as a key part of UNIX's evolution~\cite{Ritchie1984Evolution}.
Pipes are a simple mechanism that enables a deep property: \textbf{programs become compositional} because they agree on a stream interface.
This is not merely convenience; it is a portability pattern.
If programs compose through a stable stream boundary, you can swap implementations, move stages, or parallelize stages without rewriting the whole application.

\subsection{C as a Portability Amplifier}
UNIX alone did not create broad portability; UNIX + C did.
Ritchie's history of C describes C as a system implementation language devised for UNIX, evolving into a portable toolchain foundation~\cite{Ritchie1993C}.

Portability often comes from moving complexity into a compiler or runtime instead of duplicating it across applications.
C's abstract machine model made it practical to port UNIX across hardware by rewriting relatively little.
This becomes a recurring pattern:
\begin{itemize}
    \item OS portability: stable syscalls + portable implementation language (UNIX+C)~\cite{Ritchie1974UNIX, Ritchie1993C}.
    \item API portability: POSIX standardizes the boundary so many OSes can implement it~\cite{POSIX1988}.
    \item ML portability: computation graphs become an IR; runtimes map them to devices (TensorFlow, XLA-like stacks)~\cite{Abadi2016TensorFlow}.
    \item Distributed portability: tasks/actors become a portable unit; runtimes schedule them across clusters (Ray)~\cite{Moritz2018Ray}.
\end{itemize}

\subsection{Modularity and Abstraction Boundaries}
The UNIX lineage sits within a broader systems push toward \textit{explicit abstraction layers}.
Dijkstra's THE system emphasized hierarchical layers and sequential processes to structure multiprogramming systems~\cite{Dijkstra1968THE}.
Parnas's classic modularization criterion argues for decomposing systems by \textbf{information hiding}: modules should conceal design decisions likely to change, exposing stable interfaces~\cite{Parnas1972Decomposing}.

Hardware independence is ``Parnas applied to hardware'': hide the hardware-specific decisions (device access method, memory movement, scheduling strategy) behind stable module boundaries so they can evolve without rewriting the whole system.

\subsection{Deterministic Concurrency: Kahn Process Networks}
With the introduction of concurrency, ``portability of meaning'' becomes fragile: different thread schedules, queue timings, and buffering can change outcomes.
This is the historical niche filled by Kahn process networks (KPN)~\cite{Kahn1974Semantics} and related dataflow models.

\paragraph{KPN's Key Semantic Move}
Kahn's formulation defines each process as a function from input histories to output histories~\cite{Kahn1974Semantics}.
The deep idea is: if each node is deterministic in this stream-functional sense, the entire network has a well-defined meaning independent of low-level scheduling interleavings.

\paragraph{Implications for Robotics}
Robotics systems deviate from classical KPN (unbounded FIFO, blocking reads) due to real-time constraints (bounded buffers, dropping data). Retriever adapts this by:
\begin{enumerate}
    \item Treating the program as a \textit{graph of stream transducers} (KPN-like).
    \item Making buffering and synchronization policies first-class, explicit, and deterministic.
    \item Providing semantics that distinguish \textit{denotational meaning} (what traces are valid) from \textit{operational constraints} (latency bounds, scheduling).
\end{enumerate}

Synchronous Dataflow (SDF)~\cite{Lee1987SDF} further trades expressiveness for analyzability (static schedules, bounded memory).
For Retriever, SDF is evidence that \textbf{exposing graph structure and rates makes runtime optimization possible}.

\subsection{Functional Programming and Reactivity}
Backus argued that mainstream languages encourage stateful, hardware-shaped thinking, advocating instead for functional style with algebraic composition~\cite{Backus1978VonNeumann}.
Hughes emphasized that modularity comes from composing small parts through stable combinators~\cite{Hughes1989WhyFP}.

\subsubsection{Functional Reactive Programming (FRP)}
FRP introduced a disciplined way to describe reactive systems as transformations over time-varying values (signals) and discrete events~\cite{Elliott1997Fran, Nilsson2002FRP}.
Arrowized FRP explicitly connects to Hughes's Arrow framework~\cite{Hughes2000Arrows} as a theoretical underpinning for composing stateful signal functions.

ReactiveX (RxJava, RxJS) made event streams a practical systems abstraction.
Retriever adopts FRP's compositional view but makes buffering, latency, and causality explicit for robotics.
Causality is central: perception at time $t$ influences actions at time $t+\Delta$.
Making that ``delay'' explicit in the program graph is key to determinism.

\subsection{Actor Model and Message Passing}
\citet{Hewitt1973Actor} proposed the Actor model as universal primitives for concurrent computation, built around local state and asynchronous message passing.
Armstrong's Erlang thesis verified that share-nothing processes with message passing can build reliable distributed systems~\cite{Armstrong2003Erlang}.

Retriever's runtime adopts this model: each Flow is an isolated process (like an Actor), communicating via messages.
This provides \textbf{location transparency} and avoids shared-memory races, but the Actor model alone does not guarantee determinism. 
Retriever adds the KPN/FRP-inspired dataflow graph and explicit time policies to constrain the actors into a deterministic execution.

\subsection{Event Logs and Replay}
\citet{Kreps2011Kafka} emphasize durable, partitioned logs as the backbone of data pipelines.
Retriever likewise treats \textit{the log as the system of record} for deterministic replay and ``time-travel debugging.''
If inter-module communication is treated as an append-only sequence of typed events, the system becomes auditable and learnable from traces.

\subsection{Modern Learning Frameworks}
Modern machine learning frameworks like TensorFlow~\cite{Abadi2016TensorFlow} and PyTorch~\cite{Paszke2019PyTorch} demonstrate the power of representing computation as an explicit graph (or IR).
TensorFlow uses a dataflow graph to map computation across devices. PyTorch combines imperative authoring with structured execution for autograd. 

Retriever applies the same pattern to robotics.
An explicit graph and runtime contract enable determinism checks and potentially differentiable pipelines that are difficult to recover from ad-hoc pub/sub code.

\subsection{Robotics Middleware: ROS Limitations}
\citet{Quigley2009ROS} describe ROS as an open-source robot operating system emphasizing flexibility and ecosystem.
While ROS excels at integration, its pub/sub model acts as a \textit{transport} layer rather than a \textit{semantic} runtime. 
\begin{itemize}
    \item \textit{implicit structure}: The computation graph is often implicit in callbacks.
    \item \textit{ordering nondeterminism}: Many-to-many pub/sub lacks deterministic merge policies.
    \item \textit{external time}: Time semantics are often ad-hoc across nodes.
\end{itemize}
Retriever targets a complementary point: a higher-level programming model that makes the computation graph and synchronization policies explicit.

\subsection{Relation to Continuous-Time MDPs and (PO-)SMDPs (Brief)}
\label{app:formulation_related}

The \asyncmodel in Sec.~\ref{sec:formulation} emphasizes \emph{global continuous time} and \emph{asynchronous events} (sensor arrivals, compute completions, control emissions). Below we summarize a few decision-theoretic viewpoints that are relevant for interpreting this formulation:
\begin{itemize}[leftmargin=*]
    \item \textbf{Continuous-time control view (CTMDP/CTPOMDP):} the plant evolves continuously and actions may be interpreted as piecewise-constant (or envelope-constrained) signals over time. When rewards are defined as rates integrated over time, discounting naturally takes the continuous-time form $e^{-\lambda \tau}$ over a duration $\tau$.
    \item \textbf{Event-embedded view (SMDP / PO-SMDP):} define \emph{decision epochs} $\{t_k\}$ as a subset of event times (e.g., the union of observation arrivals and computation completions). Between epochs, the agent's command stream is fixed by the last emitted decision (or by a fixed low-level controller), so the state evolution over $[t_k, t_{k+1})$ induces a semi-Markov transition with holding time $\tau_k = t_{k+1}-t_k$. This yields a (partially observable) semi-Markov decision process with standard Bellman equations and existence results~\cite{Puterman1994MDP}.
    \item \textbf{Temporal abstraction (options/macro-actions):} our \emph{plan chunks} and \emph{action chunks} are naturally interpreted as temporally extended actions: they specify an initiation rule, an internal policy over a short horizon, and a termination condition. This is the classical option/SMDP connection~\cite{Sutton1999Options}.
\end{itemize}
\noindent

\subsection{Concluding Remark: A Narrow-Waist Abstraction for Robotic Agent Graphs}
Taken together with the decision-theoretic timing view in Appendix~\ref{app:formulation_related}, the lineage above motivates a narrow-waist abstraction for robotic agent computation graphs; Retriever can be viewed as a synthesis of:
\begin{itemize}
    \item From UNIX/POSIX: uniform interfaces and portable implementation~\cite{Ritchie1974UNIX, POSIX1988}.
    \item From KPN/FRP: explicit stream transduction and time semantics~\cite{Kahn1974Semantics, Nilsson2002FRP}.
    \item From Actors/Erlang: isolation and message-passing scalability~\cite{Hewitt1973Actor, Armstrong2003Erlang}.
    \item From ML/Logs: graph-based optimization and replayable history~\cite{Abadi2016TensorFlow, Kreps2011Kafka}.
\end{itemize}

This combination addresses the specific gap in robotics: building complex, closed-loop agents that are semantically rigorous yet systems-practical.
The resulting \textit{retriever semantic contract} can be summarized as:
\begin{enumerate}
    \item \textit{time model}: All messages have explicit event time; operations are defined over event time.
    \item \textit{snapshot semantics}: Explicit rules for joining/aligning asynchronous inputs at decision time.
    \item \textit{backpressure \& policy}: Buffer sizes and drop policies are explicit, not implicit.
    \item \textit{deterministic replay}: Given the same trace + policy, execution yields the same action sequence.
    \item \textit{provenance}: Every output is traceable to its input snapshot and code version.
\end{enumerate}

\section{Extended Formulation: Asynchronous Environment--Agent Loop}
\label{app:formulation_ext}

This section formalizes the \asyncmodel introduced in Sec.~\ref{sec:formulation} using the same core concepts (streams, clocks, and explicit synchronization), but with more precise definitions for the appendix proofs.

\subsection{Types: Streams and Clocks}
In this appendix, $\mathbb{T}$ denotes continuous event time (wall-clock or sim-clock); discrete decision instants are represented as event timestamps generated by clocks/triggers.

\begin{defn}[Stream]
A stream of type $V$ is a partial function $x : \mathbb{T} \rightharpoonup V$, equivalently a set of tagged events $\{(t_i,v_i)\}$ with at most one value per tag in that stream~\cite{Lee1998MoC,Liu2005TaggedSignals}.
We sometimes write $x(t)=\bot$ as shorthand for $t\notin\mathrm{dom}(x)$.
We distinguish two subtypes (following standard tagged-signal and FRP terminology)~\cite{Elliott1997Fran,Wan2000FRP,Nilsson2002FRP}:
\begin{itemize}
    \item \textbf{Behavior} $b : \mathbb{T} \to V$: defined for all $t$ (e.g., continuous physical state).
    \item \textbf{Event Stream} $e : \mathbb{T} \rightharpoonup V$: defined only at a countable, locally finite set of timestamps $\mathrm{dom}(e)$.
\end{itemize}
\end{defn}

\begin{defn}[Clock]
A clock $\mathcal{C}$ is an event stream of unitless ticks: $c : \mathbb{T} \rightharpoonup \{\mathrm{tick}\}$.
Discrete clocked execution is standard in synchronous/dataflow models~\cite{Caspi1987LUSTRE,Lee1987SDF}.
\end{defn}

\subsection{Asynchronous Environment--Agent Loop}
Standard MDPs assume the environment waits for the agent ($s_t \to a_t \to s_{t+1}$). The \asyncmodel keeps the agent--environment split but removes the single shared step:
\begin{enumerate}
    \item \textbf{Environment:} A causal map $\mathcal{E}: \mathsf{Stream}(\mathcal{A}) \to \mathsf{Stream}(\mathcal{Z})$. It evolves continuously; state at $t+\Delta$ depends on state at $t$ and actions $a|_{[t, t+\Delta]}$.
    \item \textbf{Agent:} A Causal Stream Function $\pi: \mathsf{Stream}(\mathcal{Z}) \to \mathsf{Stream}(\mathcal{A})$.
\end{enumerate}
The crucial difference is that $\pi$ has non-zero latency. If $\pi$ takes computation time $\delta$, the action $a(t)$ cannot depend on observations later than $t-\delta$. Retriever operationalizes this constraint by forcing explicit \texttt{sync} policies that define exactly which past observation $z(t-\Delta)$ is consumed.
\paragraph{Decision-theoretic relation.}
See Appendix~\ref{app:formulation_related} for a brief decision-theoretic orientation to timing and asynchronous decision epochs (CTMDPs and (PO-)SMDPs). Retriever's \texttt{Clock} + \texttt{sync} policies specify decision epochs and state snapshots, while explicit compute latency constrains which past observations are admissible at each epoch.

\section{Theoretical Guarantees}
\label{app:proof}

This appendix provides formal justifications for two core theoretical claims of the Retriever framework: (1) the graph model is sufficient for the bounded-memory causal policies considered here (Representation Theorem), and (2) the execution model yields reproducible behavior under the fixed-trace conditions of the Functional Determinism theorem. The extended environment--agent loop formalization is provided in Appendix~\ref{app:formulation_ext}.

\paragraph{Causal Stream Function (CSF).}
We model each Flow as a deterministic \emph{Causal Stream Function} (CSF): it processes timestamped input streams through a local state update.
Concretely, a CSF is a tuple $(H, h_0, f)$ where $H$ is the internal state space, $h_0 \in H$ is the initial state, and $f$ is a deterministic step map.
At each Flow sample tick $t$ (generated by a Clock or Trigger), the Flow consumes an \emph{input snapshot} $x_t$ produced by edge \texttt{sync} policies from committed upstream histories, then updates and emits:
\[
(h_{t^+}, y_t) = f(h_{t^-}, x_t).
\]
This is the locally-synchronous contract: within one \texttt{step()}, the computation is sequential and deterministic; across steps, distinct Flows may have distinct clocks and communicate through timestamped streams.
The emitted value becomes visible at its declared commit timestamp; strict-prefix visibility is the default, while exact-tick visibility requires a deterministic, acyclic same-tick order.

\subsection{Representation Theorem: Sufficiency of the Primitives}
\label{app:proof_representation}

The Retriever composition model rests on the claim that \textit{stateless maps}, \textit{stateful scans}, and \textit{synchronization policies} are sufficient to construct the bounded-memory causal robot policies considered in this paper. This mirrors constructions in synchronous languages~\cite{Caspi1987LUSTRE} and functional reactive programming~\cite{Wan2000FRP}.

\begin{thm}[Representation Sufficiency; formal version of Thm.~\ref{thm:representability_main}]
Let $\pi$ be a deterministic bounded-memory causal policy whose outputs occur at locally finite decision epochs. At each epoch, its input is a finite aligned record materialized from committed observation histories by deterministic synchronization policies, and its state lies in a fixed state space $H$ with bounded representation size. There exists a Retriever graph $G$ composed of \texttt{Map} and \texttt{Scan} Flows, clocks, and deterministic \texttt{sync} edges such that $G$ implements $\pi$.
\end{thm}

\begin{IEEEproof}[Proof Sketch]
At decision epoch $t_k$, write the policy transition directly as
\[
    (h_k,a_k)=g(h_{k-1},z_k),
\]
where $z_k$ is the aligned input record and $h_k\in H$. A clock supplies the epochs $\{t_k\}$; deterministic \texttt{sync} edges materialize each $z_k$ from the committed input histories. A \texttt{Scan} Flow stores $h_{k-1}$ and evaluates $g$, while an optional \texttt{Map} projects or transforms the emitted action. This construction permits continuous-valued fixed-size state and does not require finite cardinality.

For a multi-rate decomposition $\pi_1,\dots,\pi_n$, apply the same construction to each local clock $\mathcal{C}_i$ and connect the resulting Flows by the policy's deterministic \texttt{sync} edges:
\begin{itemize}
    \item \emph{Bounded-size memory.} The joint state space is $H=H_1\times\cdots\times H_n$. The graph contains finitely many sub-policies and each local representation size is fixed, so the composite representation size is bounded even when state values are continuous.
    \item \emph{Causality.} Each $\pi_i$ consumes only its declared committed history through \texttt{sync}; every feedback cycle contains a strict edge or delay, so no instantaneous loop is formed.
    \item \emph{Closure.} A composition of deterministic causal maps is itself a deterministic causal map.
\end{itemize}
Hence the union graph implements the same local transitions and action emissions as $\pi$. No primitive beyond clocked \texttt{Scan}/\texttt{Map} computation and deterministic \texttt{sync} is required.
\end{IEEEproof}

\subsection{Functional Determinism: Trace Reproducibility}
\label{app:proof_determinism}

Here we formalize the claim that Retriever programs are \textit{functionally deterministic}: given fixed committed input histories, clock/commit traces, and initial states, the sequence of internal states and output commands is identical across runs. This property relates to Kahn Process Networks (KPN)~\cite{Kahn1974Semantics} and Synchronous Dataflow (SDF)~\cite{Lee1987SDF}, but extends their replay viewpoint to explicit real-time policies.

\begin{defn}[Event Stream]
An event stream $S = \{(t_i, v_i)\}_{i=1}^N$ is a locally finite sequence of time-value pairs, strictly increasing in $t_i$ within that stream.
\end{defn}

\begin{defn}[Trace Determinism]
A system $\mathcal{S}$ is \textbf{trace deterministic} if, for any two runs with identical committed input histories, clock/commit traces, and initial states, the output streams are identical.
\end{defn}

\paragraph{Strict causality (delay on cycles).}
We assume cycles do not form instantaneous algebraic loops, which would impose a fixed-point constraint $x_t = f(x_t)$ rather than an executable step: any dependency around a directed cycle must pass through at least one discrete delay, so values at time $t$ depend only on events strictly earlier than $t$.
Operationally, this can be realized by state in a \texttt{Scan} Flow (which advances only on ticks), or by an explicit lag/buffer policy that guarantees the consumed value comes from event time $<t$.
This is the standard condition used to ensure a cyclic graph can be unrolled into a DAG over logical time~\cite{Lee1987SDF}.

\begin{thm}[Functional Determinism of Retriever Graphs; formalizing Thm.~\ref{thm:functional_determinism_main}]
Let $G = (V, E)$ be a directed graph of Flows (possibly cyclic). If:
\begin{enumerate}
    \item Every node $F \in V$ is a deterministic Causal Stream Function (CSF).
    \item Every edge $e \in E$ has a deterministic \texttt{sync} policy.
    \item Sample ticks, output commit timestamps, and lag decisions are fixed by deterministic policies or a replay log.
    \item Committed external-input histories and initial Flow states are fixed.
    \item \textbf{Strict Causality:} Every cycle in $G$ contains at least one delay, so that any value consumed at time $t$ depends only on events with timestamps $<t$.
    \item \textbf{Exact-tick order:} Any policy that exposes events committed at the consumer's tick declares an acyclic deterministic same-tick order; otherwise it uses strict-prefix visibility.
\end{enumerate}
Then the entire graph $G$ is trace deterministic.
\end{thm}

\begin{IEEEproof}
While $G$ may contain cycles, strict causality implies that the \textit{unrolled computation graph} over time is a DAG (i.e., we cannot have instantaneous algebraic loops $x_t = f(x_t)$). This mirrors the ``strict causality'' requirement in discrete event systems~\cite{Lee1987SDF}.
We proceed by induction over this causal order, refined by the declared same-tick order when exact-tick visibility is enabled.

\textbf{Base Case ($k=0$):} Initial states $h_0$ and the first external stream events are fixed/deterministic.

\textbf{Inductive Step:} Assume all internal states and stream values are deterministic for all events up to index $k-1$.
Consider the $k$-th event occurring at flow $F$ at time $t$.
\begin{itemize}
    \item Its inputs are derived from \texttt{sync} applied to the committed visible history. By strict causality and the acyclic same-tick rule, these inputs depend only on values earlier in the causal order.
    \item By the inductive hypothesis, these past values are identical across runs.
    \item Since $F$ and \texttt{sync} are deterministic functions, the output at event $k$ is uniquely determined.
\end{itemize}
Thus, by induction on the causal order, the entire infinite trace is unique.
\end{IEEEproof}

\subsection{Remark on Stochastic Policies and Environments}
\label{app:proof_stochastic}
The results above are stated for deterministic components, conditioned on explicit randomness streams. This conditioning covers the stochastic settings common in learned robot systems.

\textbf{Stochastic policies.}
A policy that samples its action---a Gaussian or categorical policy in RL, temperature-based decoding in a VLM, or exploration noise---factors into a deterministic sampler plus exogenous random bits: $a_k \sim \pi(\cdot \mid z_k)$ is realized as $a_k=f(z_k,\epsilon_k)$, or equivalently by a seeded pseudorandom sampler.
In Retriever, the noise source is materialized as an explicit event stream (seeds or realized draws) feeding the sampling Flow, so the graph remains a network of deterministic CSFs.
Theorem~\ref{thm:representability_main} then applies to the policy conditioned on the noise stream, and Theorem~\ref{thm:functional_determinism_main} holds with the logged noise stream included among the external inputs---which is precisely the logging-and-replay contract of Appendix~\ref{app:framework_ext}.
This recovers the standard decision-theoretic view that a stochastic policy is a deterministic policy over an input extended with exogenous noise.

\textbf{Stochastic environments.}
The environment enters every guarantee only through the realized timestamped observation trace: stochastic dynamics, sensor noise, and random event timing merely determine the distribution over traces, while each theorem is a per-trace statement.
No modification is needed; distributional claims (e.g., expected task success) follow by integrating the per-trace guarantees over the environment's trace distribution.

\textbf{Numerical vs.\ event-level determinism.}
Trace determinism concerns the \emph{event path}: which ticks fire, which input snapshots are consumed, and which values are emitted, given the external input streams.
It does not promise bitwise-identical floating point across live runs: GPU kernels may reduce in nondeterministic order, exactly as in deep-learning frameworks.
Automatic differentiation still differentiates the \emph{realized} computation, so forward values and gradients remain consistent.
The path-gradient experiment in Sec.~\ref{sec:eval} therefore requires only a fixed event path.
Retriever fixes the discrete event structure that scheduling races would otherwise change.
Within-kernel numerics are orthogonal; deterministic kernels can recover bitwise reproducibility when needed.
Floating-point drift in a live run could still flip a downstream threshold event.
When the relevant consumed values are logged, replay prevents this divergence by reading them from the log.
\emph{Value-level} nondeterminism from hosted model APIs is different: outputs that vary for the same prompt must be logged and replayed as values (Sec.~\ref{sec:framework}).

\subsection{Remark on Real-Time Constraints}
Arrival-driven callbacks can let network jitter change which value a step reads.
Retriever instead makes timestamp, freshness, and deadline rules explicit in \texttt{sync}.
Under the fixed replay conditions of Thm.~\ref{thm:functional_determinism_main}, late data has a declared stale, missing, dropped, or timeout outcome.
This separates the program's \textit{denotational semantics} from its \textit{operational execution}, a core tenet of synchronous languages~\cite{Caspi1987LUSTRE}.

\section{Retriever Runtime Implementation}
\label{app:framework_ext}
This appendix provides API examples and operational details omitted from Sec.~\ref{sec:framework}.
It focuses on \emph{what the user writes}, \emph{what the runtime guarantees}, and \emph{what gets logged for replay}.

\subsection{Runtime Design (Flow / IR / Runtime Layers)}
Sec.~\ref{sec:framework} gives the authoring $\rightarrow$ IR $\rightarrow$ runtime path.
Here, a user-authored Pipeline compiles to a static, typed IR that fixes topology, clocks, edge policies, and backend adapters before execution.

\paragraph{IR example (serialized contract).}
The IR is an explicit, typed, serializable graph object: it fixes node types, clocks, and edge adapters, and can be saved/loaded as a deployment artifact (no Python source required at runtime). A schematic excerpt is:
\begin{lstlisting}[style=retriever, language=Python]
ir = pipe.build_ir()          # validates types, clocks, adapters
ir.save("agent.ir.json")      # deployable artifact

# Schematic excerpt of ir.to_json() (fields omitted):
{
  "version": "1.0.0",
  "metadata": {"name": "Agent", "validated": True},
  "nodes": [
    {"id": "cam", "type": "CameraSource",
     "config": {"clock": {"Rate": {"hz": 30}}},
     "outputs": {"frame": "Image"}}
  ],
  "edges": [
    {"source": {"node": "cam", "port": "frame"},
     "destination": {"node": "skill", "port": "frame"},
     "adapter": {"Latest": {"staleness": 0.1}},
     "qsize": 10}
  ],
  "topology": {"has_cycle": False}
}
\end{lstlisting}
\noindent
In practice, the full IR includes port-level input/output types, queue/backpressure settings, and the computed topology metadata; it underpins visualization (\texttt{IR.visualize()}) and backend compilation (\texttt{IR.compile()}).

\paragraph{Scheduling, backpressure, and missed rates.}
Real pipelines may miss a requested rate under load, for example during slow VLM calls or heavy perception.
Each Flow declares its desired clock, so the runtime can report sustained misses and either terminate (hard real-time), tolerate a bounded number before escalation (soft real-time), or warn and continue (best-effort).
These policies configure the runtime; they do not change program semantics.

\subsection{Flow Interface and Run Clocks}
A \texttt{Flow} is a synchronous state machine stepped by an explicit \textit{run clock}. Users implement (i) state initialization and (ii) a pure step on aligned inputs:
\begin{lstlisting}[style=retriever, language=Python]
class BeliefMemoryFlow(Flow[(Obs, Event), Belief]):
    def reset(self):
        self.belief = Belief()

    def step(self, inp):
        obs, event = inp
        if event == "inspection_done":
            self.belief = update_belief(self.belief, obs)
        return self.belief

belief = BeliefMemoryFlow() @Trigger("inspection_done")
\end{lstlisting}

\subsection{Edgewise Synchronization Policies (How Inputs Are Sampled)}
The formal \texttt{sync} contract and concrete \texttt{Latest}/\texttt{Window} examples are in Sec.~\ref{subsec:framework_sync}. The corresponding user-facing wiring is simply an edge declaration:
\begin{lstlisting}[style=retriever, language=Python]
pipe = Pipeline("Agent")
with pipe:
    cam.then(vla,    sync=Latest())\
       .then(ctrl,   sync=Latest())  # controller plays ActionChunk over fast ticks
    cam.then(belief, sync=Latest())\
       .then(monitor,sync=Latest())
\end{lstlisting}

\paragraph{Why This Matters (Pub/Sub Failure Mode)}
In pub/sub systems, input alignment is often an implicit ``latest message'' rule that can vary across runs.
Retriever makes each sampled snapshot explicit through \texttt{sync}.
A complete replay record includes the consumed input IDs and the clock/commit trace required by Thm.~\ref{thm:functional_determinism_main}.

\subsection{Worked Example: Case Study Wiring and Monitor-Mediated Plan Updates}
\begin{figure*}[t]
    \centering
    \includegraphics[width=0.9\linewidth]{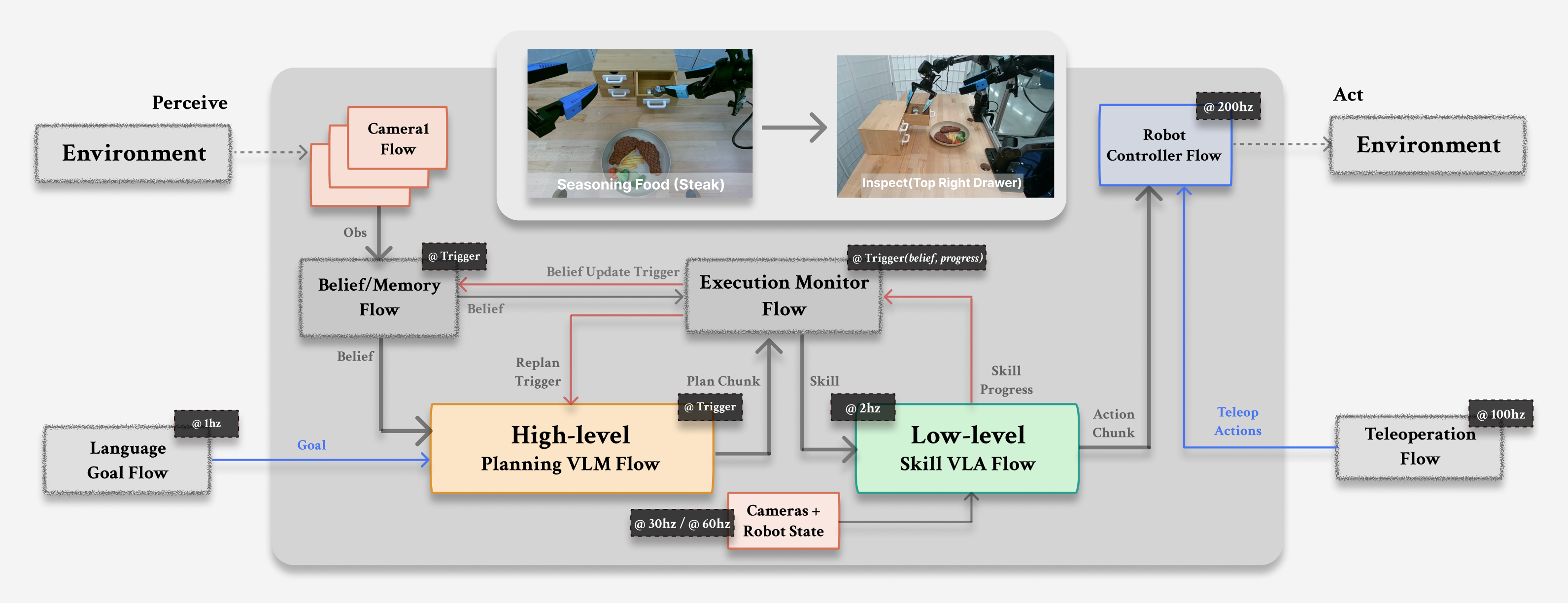}
    \caption{\textbf{Representative closed-loop hierarchical manipulation pipeline (appendix view).} The same case-study graph from Sec.~\ref{sec:case_study} is shown with the messages that matter for implementation review: predicate-level belief, \texttt{PlanChunk} proposals, one active \texttt{SkillCmd}, \texttt{ActionChunk}s, progress predictions, and terminal events.}
    \label{fig:app_case_study_pipeline}
    \vspace{-1.2em}
\end{figure*}
The case-study agent (Sec.~\ref{sec:case_study}; Fig.~\ref{fig:app_case_study_pipeline}) combines slow planning and belief updates with fast skill and control loops.
At the appendix level, the planner proposes bounded-horizon \texttt{PlanChunk}s, the monitor commits them at handoff boundaries, and the controller consumes \texttt{ActionChunk}s on a faster clock.
Appendix~\ref{app:case_study_ext} records the message interfaces.

\subsection{Temporal Chunking (Plan and Action)}
Figure~\ref{fig:temporal_chunking} introduces the shared chunking pattern.
An \texttt{ActionChunk} buffers VLA output for 200Hz control, while a \texttt{PlanChunk} lets the monitor overlap replanning with the current skill.
Only action chunks are generic buffering adapters; plan chunks remain part of the execution-monitor policy in Appendix~\ref{app:case_study_ext}.

\subsection{Runtime Modes: Debug vs Deployment}
Retriever compiles a user graph to a static IR of nodes, edges, clocks, and policies, then runs it on a supported backend:
\begin{itemize}[leftmargin=*]
    \item \textbf{In-process stepping:} deterministic, debugger-friendly execution (\texttt{pipe.step}) for development and trace inspection.
    \item \textbf{Asynchronous execution:} actors/processes for deployment (\texttt{pipe.run}), where each Flow executes sequentially on its clock and edges implement deterministic policies.
\end{itemize}

\subsection{Logging for Deterministic Replay}
Exact replay requires logging each step's event time, Flow tick ID, \emph{consumed input IDs}, clock/commit decisions, and stochastic seeds or outputs.
These records identify the snapshot selected by each \texttt{sync} policy and allow offline re-execution (Sec.~\ref{sec:eval}, Appendix~\ref{app:eval_ext}).

\subsection{Pipeline Interface Summary}
\label{app:code_examples}

The released implementation follows the same six-Flow decomposition as Sec.~\ref{sec:case_study}. We summarize the interface contract here rather than reproducing source-level code; the concrete class names and a runnable pipeline are available in the released repository (\url{https://github.com/openretriever/retriever}).

Appendix~\ref{app:case_study_ext} gives the concrete per-Flow interfaces and trigger chain. The framework-level takeaway is only that all messages are typed, timestamped records flowing through explicit clocks and \texttt{sync} policies.

\subsection{VLM Prompting Details}
\label{app:prompts}

To ensure the VLM planner emits plans that the monitor can execute, we constrain the prompt around three ingredients: a finite object/affordance vocabulary, predicate-style belief facts, and a closed skill vocabulary with cleanup/termination rules. These prompt constraints produce bounded-horizon skill programs whose branch conditions can be resolved by the monitor from the current belief.

\paragraph{Belief updater prompt.}
The belief updater similarly maps observations plus the previous belief to predicate-valued facts. We use a three-valued style (\texttt{yes/no/unknown}) and avoid overwriting known facts with \texttt{unknown} unless the task evidence warrants it.

\paragraph{Illustrative prompt contract (drawer task).}
The important interface is the typed belief and plan contract the prompt induces, not its exact wording.
In the drawer task, the belief updater can be viewed schematically as maintaining three-valued predicates over objects and containers:
\begin{lstlisting}[style=retriever, language=Python, numbers=none, basicstyle=\ttfamily\scriptsize, xleftmargin=4pt, framexleftmargin=4pt, aboveskip=0.4em, belowskip=0.4em]
contains(top_left_drawer, spice)  in {yes, no, unknown}
contains(top_right_drawer, spice) in {yes, no, unknown}
holding(spice)                    in {yes, no, unknown}
seasoned(food)                    in {yes, no, unknown}
\end{lstlisting}
The planner consumes a goal and this compact belief state, then emits a bounded \texttt{PlanChunk} over the closed skill vocabulary. A simplified example is:
\begin{lstlisting}[style=retriever, language=Python, numbers=none, basicstyle=\ttfamily\scriptsize, xleftmargin=4pt, framexleftmargin=4pt, aboveskip=0.4em, belowskip=0.4em]
Goal: season(food) with spice
Belief: contains(top_left_drawer, spice) = unknown

PlanChunk:
  OpenDrawer(top_left_handle)
  InspectDrawer(top_left_drawer)
  IF contains(top_left_drawer, spice) = yes:
      Pick(spice)
      Season(food)
  ELSE:
      CloseDrawer(top_left_handle)
      OpenDrawer(top_right_handle)
      InspectDrawer(top_right_drawer)
\end{lstlisting}
This example illustrates the belief-space predicate interface used by the monitor.

\section{Case Study Details}
\label{app:case_study_ext}
This appendix records case-study specific interfaces and conventions referenced in Sec.~\ref{sec:case_study} (belief representation, \texttt{PlanChunk} structure, skill catalog, teleoperation schema, and task-progress scoring). It complements Appendix~\ref{app:framework_ext} by focusing on concrete messages and task-level conventions rather than generic runtime semantics.
We use compact aliases F1--F6 below only as appendix shorthand: F1=$F_{\texttt{cam}}$, F2=$F_{\texttt{belief}}$, F3=$F_{\texttt{plan}}$, F4=$F_{\texttt{exe}}$, F5=$F_{\texttt{skill}}$, and F6=$F_{\texttt{ctrl}}$.

\paragraph{Flow Specifications (Inputs/Outputs/Triggers)}
\begin{itemize}
    \item \textbf{F1 Observation ($F_{\texttt{cam}}$)}: emits timestamped observations (camera + optional robot state) at 30Hz (or sensor-driven).
    \item \textbf{F2 Belief/Memory ($F_{\texttt{belief}}$)}: updates only on inspection completion (\texttt{inspection\_done}). Tracks a compact task state (e.g., drawer \texttt{\{unknown, checked, target\_seen\}} and object \texttt{\{unknown\_location, in\_drawer(d), in\_gripper, done\}}). Emits \texttt{belief\_updated} (and optional \texttt{belief\_delta}).
    \item \textbf{F3 Planning VLM ($F_{\texttt{plan}}$)}: triggered by the monitor on \texttt{belief\_updated}, failures, or when the remaining plan horizon is low. Inputs: belief summary, goal, and execution context. Output: a bounded-horizon \texttt{PlanChunk} of up to five skills, often a small conditional tree whose branch conditions are predicates over belief/evidence. In our implementation, this component is an external multimodal planner service wrapped as a Retriever Flow.
    \item \textbf{F4 Execution Monitor ($F_{\texttt{exe}}$)}: stores the active \texttt{PlanChunk} subtree, the current node pointer, latest belief, and current skill state. Triggered by \texttt{skill\_done/skill\_failed}, \texttt{inspection\_done}, \texttt{PlanChunk} proposals, and progress-prediction events; optional 5--10Hz heartbeat for timeouts and horizon checks. Outputs the active \texttt{SkillCmd} to F5 and replan triggers to F3.
    \item \textbf{F5 Skill VLA ($F_{\texttt{skill}}$)}: executes the active \texttt{SkillCmd} nominally at 2Hz and outputs \texttt{ActionChunk}s to F6 plus progress prediction to F4. Inspection skills additionally emit \texttt{inspection\_done}.
    \item \textbf{F6 Controller ($F_{\texttt{ctrl}}$)}: consumes \texttt{ActionChunk}s and runs a safety-critical 200Hz control loop.
\end{itemize}

\paragraph{Low-level policy data efficiency.}
Across the drawer and bag-retrieval tasks, the shared low-level policy uses roughly \textbf{200--300 complete expert episodes total}.
Each episode lasts roughly \textbf{1--3 minutes}, yielding only a \textbf{few hours of robot data}.
The dataset covers the closed skill vocabulary rather than separate large per-skill datasets.

\paragraph{Low-level VLA training details ($\pi_{0.5}$ fine-tuning or multi-task diffusion).}
We found two training routes sufficient for $F_{\texttt{skill}}$: (i) fine-tune the $\pi_{0.5}$ VLA model~\cite{pi05} for the closed skill vocabulary, or (ii) train a multi-task diffusion policy from scratch on the same dataset.
Both routes condition the shared policy on the current \texttt{SkillCmd} (skill name + arguments) and use the union of skill-labeled expert trajectories.

Skills such as \texttt{OpenDrawer}, \texttt{InspectDrawer}, \texttt{CloseDrawer}, \texttt{Pick}, and \texttt{PlaceBack} share the same observation and action interfaces.
Because both training routes support this skill set, the system pipeline is not tied to one low-level learning recipe.

\paragraph{Monitor-Mediated Plan Updates (Plan Chunking in the Monitor)}
In this case study, \emph{plan chunking lives inside $F_{\texttt{exe}}$}, not in an edge policy.
The monitor never interrupts the active skill.
A new \texttt{PlanChunk} can replace only the remaining suffix, allowing replanning to overlap execution.
Belief updates, failures, and a short remaining horizon trigger new proposals.
\paragraph{What the planner emits (skills, not actions)}
\texttt{PlanChunk} is a short bounded-horizon \emph{skill program} over a closed skill vocabulary (Sec.~\ref{sec:case_study}).
Each node is a parameterized macro-action such as \texttt{OpenDrawer(handle)}, \texttt{InspectDrawer(container)}, \texttt{Pick(object)}, \texttt{PlaceBack(object)}, or \texttt{CloseDrawer(handle)}.
The planner does not emit low-level control.
The VLA policy consumes one active \texttt{SkillCmd} at a time and handles within-skill control through \texttt{ActionChunk}s; the monitor handles skill handoff.
In practice, the \texttt{PlanChunk} places branch points at information-gathering skills (e.g., inspection) so that belief updates can resolve partial observability online. A representative chunk for the drawer task is:
\begin{lstlisting}[style=retriever, language=Python]
# Example PlanChunk (drawer task, schematic)
OpenDrawer(top_left_handle) -> InspectDrawer(top_left_drawer)
IF contains(top_left_drawer, spice) = yes THEN
    Pick(spice) -> Season(food) -> PlaceBack(spice)
ELSE
    CloseDrawer(top_left_handle)
\end{lstlisting}
$F_{\texttt{exe}}$ resolves the IF/THEN/ELSE using the latest belief snapshot and issues the next concrete \texttt{SkillCmd} to $F_{\texttt{skill}}$ at explicit handoff points.

\paragraph{Conditional Execution (Belief-Space IF/THEN/ELSE)}
To handle partial observability, the planner outputs short programs that can branch after inspections, e.g.,
\texttt{IF contains(drawer, target) = yes THEN pick(target) ELSE close-drawer(drawer)}.
$F_{\texttt{exe}}$ resolves these branches at runtime using the latest belief snapshot (which is refreshed by inspection-triggered belief updates). This keeps the VLA interface simple: $F_{\texttt{skill}}$ always receives a concrete \texttt{SkillCmd}, never a conditional.

\paragraph{Progress Prediction for Handoff}
The skill output includes a progress prediction for monitor-mediated handoff.
The signal is timestamped, synchronized, and logged like other streams; the predictor and scoring rule are implementation choices behind this interface.

\paragraph{Planner Prompting (Object Grounding \& Completion Rules)}
In practice, VLM planners can hallucinate objects or omit critical cleanup steps. We use prompt constraints that (i) define the complete object set (four drawers + known spice objects), and (ii) enforce task-completion rules (e.g., always close drawers; always terminate after seasoning). These constraints improve reliability without changing the runtime.

\paragraph{Rates and Time-Semantics (Where \texttt{sync} Applies)}
The case study uses 30Hz cameras, a nominally 2Hz VLA, and a 200Hz controller.
At each VLA tick, \texttt{sync} samples the perception and state streams.
Progress-prediction events trigger the monitor, while the controller consumes the latest valid \texttt{ActionChunk} on its own clock.

\paragraph{Teleoperation for HIL \& Data Collection}
Teleop is an optional input Flow for intervention, recovery, and data collection.
After a failure, the monitor can request another inspection or replan, or it can accept a temporary teleop override.
Teleop actions are logged alongside the event stream for imitation data and post-hoc analysis.

\paragraph{Concrete Interfaces (Messages)}
The pipeline logs observations, \texttt{SkillCmd}s, \texttt{PlanChunk}s, \texttt{ActionChunk}s, progress predictions, and terminal events as typed messages.
The monitor sends one concrete \texttt{SkillCmd} to the VLA and receives progress predictions; the controller consumes \texttt{ActionChunk}s on its own clock.

\paragraph{Skill Catalog and Object Affordances}
In the drawer task we separate the drawer \emph{container} from its \emph{handle} so skills can be parameterized by the correct affordance. For readability, examples use explicit names such as \texttt{top\_left\_drawer} and \texttt{top\_left\_handle} rather than short drawer codes. A representative skill library includes:
\texttt{OpenDrawer(top\_left\_handle)}, \texttt{InspectDrawer(top\_left\_drawer)}, \texttt{CloseDrawer(top\_left\_handle)}, \texttt{Pick(spice)}, \texttt{Season(food)}, and \texttt{PlaceBack(spice)} (plus recovery variants).
The bag task uses an analogous set of skills (e.g., \texttt{OpenBag}, \texttt{InspectBag}, \texttt{Pick(target)}, \texttt{Extract(target)}), but the same monitor/planner interfaces.

\paragraph{Teleoperation Logging}
Teleop can be used both as a safety valve and as a data source. For each intervention, the log records the intervention window and overridden skill/action. Under the replay contract above, recording the consumed input identifiers also pairs each teleop action with the policy/monitor snapshot needed for imitation learning and debugging.

\paragraph{Task Progress and Projected Time}
We report normalized task progress $p\in[0,1]$ from the execution monitor's task milestones and projected total time $\widehat{T}=T_{\mathrm{best}}/p$ (Sec.~\ref{sec:eval}). The projection is a descriptive linear normalization, not a measured completion time.
Drawer milestones cover the eight-skill trace in Sec.~\ref{sec:eval}; bag retrieval uses the four-stage sequence \texttt{OpenBag}, \texttt{InspectBag}, \texttt{Pick}, and \texttt{Extract}. Thus, a run that reaches approximately the end of the first skill receives $1/8$ drawer progress or $1/4$ bag progress. More complete runs use the monitor's finer normalized score recorded at termination.
This metric distinguishes partial progress from total failure while retaining the raw achieved-progress value beside the projected time.

\paragraph{Mock-to-Robot Portability}
The belief updater, planner, monitor, and goal interface stay unchanged between mock and robot deployments.
Only perception and action Flows are swapped, such as an image-folder camera for real cameras and a mock policy for the VLA and controller.
This lets most pipeline and prompt bugs be debugged before hardware runs.

\paragraph{Latency Profile (Typical)}
The planner has seconds-level model latency, while belief updates are faster and the monitor is a state machine.
Monitor-mediated handoff and action chunking keep this slow inference off the 200Hz control loop.

\section{Experiment Details}
\label{app:eval_ext}
\paragraph{Reproducibility.}
The runtime is released as \texttt{retriever-core} (Python~3.11+; \texttt{pip install retriever-core}), and the latency study uses \texttt{dora-rs}~$\geq$~0.3.12.
Recorded node streams can be stored in MCAP, optionally visualized with Rerun, and replayed as graph inputs.
Exact trace replay also requires the consumed snapshots and clock/commit trace in Thm.~\ref{thm:functional_determinism_main}.
Each experiment documents its hardware, software versions, execution flags, and trial protocol.

\begin{figure*}[t]
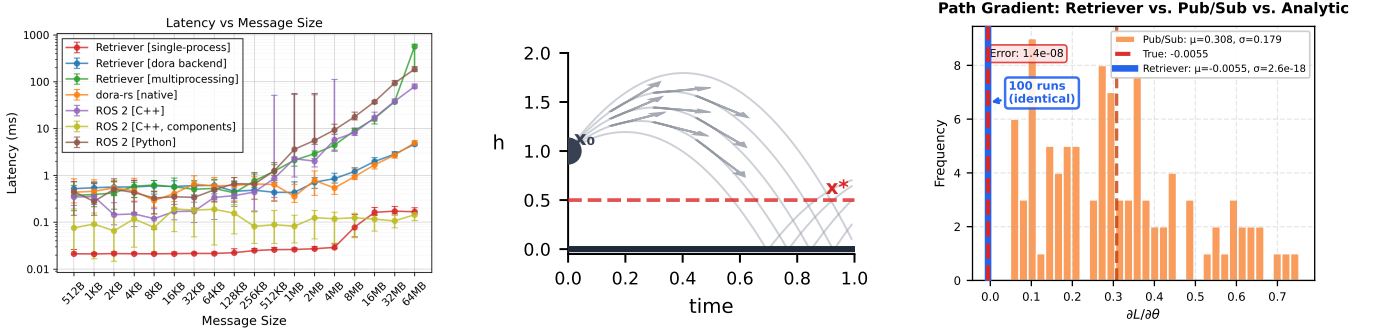

    \centering
    \begin{minipage}[t]{0.32\linewidth}
        \centering
        \includegraphics[width=\linewidth]{figures/results/combined-latency-results-new.png}
        \vspace{-0.5em}
    \end{minipage}\hfill
    \begin{minipage}[t]{0.30\linewidth}
        \centering
        \includegraphics[width=\linewidth]{figures/illustration/diff_physics_viz.pdf}
        \vspace{-0.5em}
    \end{minipage}\hfill
    \begin{minipage}[t]{0.32\linewidth}
        \centering
        \includegraphics[width=\linewidth]{figures/illustration/gradient_histogram.pdf}
        \vspace{-0.5em}
    \end{minipage}
    \vspace{-1.0em}
    \caption{\textbf{Efficiency and determinism visualizations (extended version of the main-text figure).} (a) End-to-end message latency versus payload size (log--log). This benchmark isolates transport and runtime overhead for a minimal producer-consumer pair, repeated over payload sizes and trials. (b) Hybrid differentiable physics example used to stress trace sensitivity under asynchronous semantics. (c) Distribution of \emph{path gradients} across repeated runs: event-time semantics yield a single trace and a single gradient, while arrival-time semantics can produce different traces (due to jitter-induced staleness) and hence inconsistent gradients.}
    \label{fig:app_results_benchmark}
    \vspace{-1.2em}
\end{figure*}

\paragraph{Message Latency Benchmark (Dora Backend)}
We measure end-to-end producer-to-consumer latency across payload sizes (Fig.~\ref{fig:app_results_benchmark}(a)); error bars show variation across trials.
The comparison includes Retriever in one process, across processes, and with Dora, alongside native \texttt{dora-rs} and ROS 2 C++/Python baselines.
The ROS 2 C++ components curve uses single-process composable nodes~\cite{Macenski2023ROS2Composition}, so it should be compared with \texttt{Retriever} single-process rather than multi-process ROS 2.
Retriever with Dora closely matches native \texttt{dora-rs}, while single-process Retriever gives the lowest small-message latency among the measured stacks.
These results suggest that the graph IR, clocks, and policy layer add little overhead relative to transport in this benchmark.

\paragraph{Functional Determinism Experiment (Hybrid Bouncing Ball)}
We compare event-time semantics (Retriever) with arrival-time semantics (pub/sub) on a hybrid system with continuous dynamics and discrete impacts.
The learnable parameter is initial velocity $\theta$, and we report the \textit{path gradient} $\nabla_\theta L$ for each run rather than an expectation gradient.

\paragraph{Why a Hybrid Physics Example? (Path vs.\ Expectation Gradients)}
Hybrid dynamics make the executed computation graph sensitive to the impact tick.
A one-tick shift can change the differentiated path, making this a direct test of runtime semantics.
We study \emph{path gradients} of the realized trace:
\[
g_{\text{path}} = \nabla_\theta L(\tau_{\text{run}}).
\]
If arrival-time staleness changes $\tau_{\text{run}}$ under identical initial conditions, the backward pass yields inconsistent gradients.
Expectation gradients over scheduling randomness require a different stochastic treatment and are outside our scope.

\paragraph{Hybrid System}
Free-flight dynamics with gravity:
\[
v_{t+1} = v_t - g\,dt,\qquad x_{t+1} = x_t + v_{t+1}\,dt
\]
Impact guard ($x<0$): $x \leftarrow 0$, $v \leftarrow -e\,v$, with restitution $e\in(0,1)$. Loss: $L=(x_T-x_{\text{target}})^2$.

\paragraph{Executors}
\textbf{Retriever (event-time):} fixed logical clock $t=0,\dots,T-1$, deterministic edge policies, fixed topological execution order.
\textbf{Pub/Sub (arrival-time):} a generic latest-arrived baseline; random jitter yields stale reads that can shift impact ticks. This is not a live ROS 2 or Dora determinism measurement.

\paragraph{Standard Configuration}
Horizon $T=200$, $\theta=3.0$, restitution $e=0.8$, time step $dt=0.01$, jitter probability $p=0.2$, $K=100$ runs.
\textit{Retriever}: 1 unique trace; 1 unique gradient; $\sigma_g \approx 3.5\times10^{-18}$.
\textit{Pub/Sub}: 100 unique traces; 35 unique gradients; $\sigma_g \approx 0.177$; 99\% trace mismatch.

\paragraph{High-Jitter Configuration}
Horizon $T=150$, $\theta=4.0$, $e=0.85$, $dt=0.01$, $p=0.4$, $K=50$.
\textit{Retriever}: 1 unique trace; deterministic gradient.
\textit{Pub/Sub}: 50 unique traces; $\sigma_g \approx 0.110$; 98\% mismatch; gradient sign can flip.

\paragraph{Interpretation}
Arrival-time semantics can shift impact ticks through stale reads, changing both the hybrid trace and its gradient.
Event-time semantics yield identical traces and gradients across runs.
Figure~\ref{fig:app_results_benchmark}(c) shows the gradient distribution.

\subsection{Benchmark Implementation Footprint}
\label{app:dev_complexity_raw_code}
We report approximate user-script sizes for the latency benchmark in Fig.~\ref{fig:app_results_benchmark}(a).

\paragraph{Approximate size of user scripts}
\begin{center}
\begin{tabular}{l c}
\toprule
\textbf{Stack} & \textbf{User-script size} \\
\midrule
Retriever (single benchmark script) & 117 \\
Dora (publisher + subscriber scripts) & 32 + 53 \\
ROS2 (publisher + subscriber scripts) & 79 + 95 \\
\bottomrule
\end{tabular}
\end{center}
\noindent
The benchmark scripts all implement the same producer--consumer latency sweep: timestamp a payload at the producer, deliver it through the corresponding runtime/backend, record receive time at the consumer, and aggregate latency by payload size.

\textbf{Dora benchmark scripts (native Python API)}

% Source (released baseline): node_1.py
\begin{lstlisting}[style=retriever, language=Python, caption={Dora publisher (native Python API; excerpt).}]
#!/usr/bin/env python
# -*- coding: utf-8 -*-

import time

import numpy as np
import pyarrow as pa
from dora import Node

SIZES = [2**i for i in range(6, 25)]

node = Node()
pa.array([])  # pyarrow warm-up

for size in SIZES:
    for _ in range(0, 100):
        now = time.time()
        random_data = np.random.randint(1000, size=size, dtype=np.uint64)
        random_data[0] = time.perf_counter_ns()

        node.send_output("latency", pa.array(random_data))
        time.sleep(max(0, 0.05 - (time.time() - now)))

node.send_output("latency", pa.array([], type=pa.uint64()))
\end{lstlisting}

% Source (released baseline): node_2.py
\begin{lstlisting}[style=retriever, language=Python, caption={Dora subscriber (native Python API; excerpt).}]
#!/usr/bin/env python
# -*- coding: utf-8 -*-

import time

import pyarrow as pa
from dora import Node
from helper import record_results

pa.array([])  # pyarrow warm-up
node = Node()

current_size = 8
n = 0
i = 0
latencies = []

NAME = "dora Node"

while True:
    event = node.next()
    if event["type"] == "INPUT":
        data = event["value"]
    else:
        break
    
    if len(data) == 0:
        break

    t_received = time.perf_counter_ns()
    length = len(data) * 8
    if length != current_size:
        if n > 0:
            record_results(NAME, current_size, latencies)
        current_size = length
        n = 0
        start = time.perf_counter_ns()
        latencies = []

    t_send = data[0].as_py()
    latencies.append((t_received - t_send) / 1000)

    n += 1
    i += 1

record_results(NAME, current_size, latencies)
\end{lstlisting}

\textbf{ROS2 benchmark scripts (Python)}

% Source (released baseline): publisher_member_function.py (derived from ROS2 minimal pub/sub examples)
\begin{lstlisting}[style=retriever, language=Python, caption={ROS2 publisher (Python; excerpt).}]
# Copyright 2016 Open Source Robotics Foundation, Inc.
#
# Licensed under the Apache License, Version 2.0 (the "License");
# you may not use this file except in compliance with the License.
# You may obtain a copy of the License at
#
#     http://www.apache.org/licenses/LICENSE-2.0
#
# Unless required by applicable law or agreed to in writing, software
# distributed under the License is distributed on an "AS IS" BASIS,
# WITHOUT WARRANTIES OR CONDITIONS OF ANY KIND, either express or implied.
# See the License for the specific language governing permissions and
# limitations under the License.

import time

import numpy as np
import rclpy
from rclpy.node import Node
from std_msgs.msg import UInt64MultiArray

SIZES = [
    8,
    64,
    512,
    10 * 512,
    100 * 512,
    1000 * 512,
    10000 * 512,
    8,
]


class MinimalPublisher(Node):
    def __init__(self):
        super().__init__("minimal_publisher")
        self.publisher_ = self.create_publisher(UInt64MultiArray, "topic", 10)
        timer_period = 0.05  # seconds
        self.timer = self.create_timer(timer_period, self.timer_callback)
        self.i = 0
        self.j = 0

    def timer_callback(self):
        msg = UInt64MultiArray()
        random_data = np.array(
            np.random.randint(255, size=SIZES[self.i], dtype=np.uint64)
        )

        random_data[0] = np.array([time.perf_counter_ns()])

        random_data = random_data.tobytes()
        msg.data.frombytes(random_data)
        self.publisher_.publish(msg)
        if self.j == 100:
            self.i += 1
            self.j = 0
        else:
            self.j += 1


def main(args=None):
    rclpy.init(args=args)

    minimal_publisher = MinimalPublisher()
    rclpy.spin(minimal_publisher)

    minimal_publisher.destroy_node()
    rclpy.shutdown()


if __name__ == "__main__":
    main()
\end{lstlisting}

% Source (released baseline): subscriber_member_function.py (derived from ROS2 minimal pub/sub examples)
\begin{lstlisting}[style=retriever, language=Python, caption={ROS2 subscriber (Python; excerpt).}]
# Copyright 2016 Open Source Robotics Foundation, Inc.
#
# Licensed under the Apache License, Version 2.0 (the "License");
# you may not use this file except in compliance with the License.
# You may obtain a copy of the License at
#
#     http://www.apache.org/licenses/LICENSE-2.0
#
# Unless required by applicable law or agreed to in writing, software
# distributed under the License is distributed on an "AS IS" BASIS,
# WITHOUT WARRANTIES OR CONDITIONS OF ANY KIND, either express or implied.
# See the License for the specific language governing permissions and
# limitations under the License.

import csv
import os
import time

import numpy as np
import rclpy
from rclpy.node import Node
from std_msgs.msg import UInt64MultiArray

LATENCY = True

NAME = os.getenv("NAME") or "ROS 2"
PLATFORM = "COMPUTER_PERF"
current_size = 8
n = 0
latencies = []
save_x = []


def record_results(start, current_size, latencies, latency: bool):
    avg_latency = np.array(latencies).mean()

    csv_file = os.getenv("CSV_TIME_FILE", "time.csv")
    append = os.path.isfile(csv_file)
    log_header = ["name", "platform", "size", "latency"]
    log_row = [NAME, PLATFORM, current_size, avg_latency]
    if append:
        with open(csv_file, "a", encoding="utf-8") as f:
            w = csv.writer(f, lineterminator="\n")
            w.writerow(log_row)
    else:
        with open(csv_file, "w+", encoding="utf-8") as f:
            w = csv.writer(f, lineterminator="\n")
            w.writerow(log_header)
            w.writerow(log_row)


class MinimalSubscriber(Node):
    def __init__(self):
        super().__init__("minimal_subscriber")
        self.subscription = self.create_subscription(
            UInt64MultiArray, "topic", self.listener_callback, 10
        )
        self.subscription  # prevent unused variable warning
        self.current_size = 0
        self.latencies = []
        self.n = 0

    def listener_callback(self, msg: UInt64MultiArray):
        t_received = time.perf_counter_ns()
        length = len(msg.data) * 8  # As it is Uint64
        if length != self.current_size:
            if self.n > 0:
                record_results([], self.current_size, self.latencies, LATENCY)
            self.current_size = length
            self.n = 0
            self.latencies = []
        t_send = msg.data[0]
        self.latencies.append((t_received - t_send) / 1000)
        self.n += 1


def main(args=None):
    rclpy.init(args=args)

    minimal_subscriber = MinimalSubscriber()
    rclpy.spin(minimal_subscriber)

    minimal_subscriber.destroy_node()
    rclpy.shutdown()


if __name__ == "__main__":
    main()
\end{lstlisting}

\end{adjustwidth}

\end{document}